\documentclass[12pt, a4paper]{article}

\usepackage[margin=1.19in]{geometry}
\usepackage[utf8]{inputenc} 
\usepackage[T1]{fontenc}    

\usepackage[colorlinks=true,linkcolor=blue,citecolor=blue]{hyperref}
\usepackage{url}            
\usepackage{booktabs}      
\usepackage{nicefrac}     
\usepackage{algorithm}
\usepackage{algorithmic} 
\usepackage{adjustbox}
\usepackage{amsmath, amsfonts, amsthm, amssymb} 
\usepackage{times}
\usepackage{latexsym}
\usepackage{graphicx}
\usepackage{wrapfig}
\usepackage{comment} 
\usepackage[inline]{enumitem}
\usepackage{caption,subcaption}
\usepackage[english]{babel} 
\usepackage{xcolor}
\usepackage{xspace}
\usepackage{amssymb}
\usepackage{thmtools}
\usepackage{mathtools}
\usepackage{nameref,hyperref,cleveref}
\usepackage[square,sort,comma]{natbib}
\usepackage{multicol,multirow}
\usepackage{authblk}

\setlist[enumerate]{leftmargin=*}
\setlist[itemize]{leftmargin=*}
\setlist{nosep}

\input{format}
\def \gail#1 {{\color{purple} [Gail: #1] }}
\newcommand{\ours}{\textit{AdaWAC}\xspace}
\newcommand{\best}[1]{\textbf{#1}}

\title{\Large Adaptively Weighted Data Augmentation Consistency Regularization for Robust Optimization under Concept Shift}

\author[1]{Yijun Dong\thanks{Equal contribution. Correspondence to: \href{mailto:ydong@utexas.edu}{ydong@utexas.edu}}}
\author[2]{Yuege Xie$^*$\thanks{Work done at University of Texas at Austin.}}
\author[1]{Rachel Ward}
\affil[1]{University of Texas at Austin}
\affil[2]{Snap Inc.}

\begin{document}

\maketitle

\begin{abstract}
Concept shift is a prevailing problem in natural tasks like medical image segmentation where samples usually come from different subpopulations with variant correlations between features and labels. One common type of concept shift in medical image segmentation is the ``information imbalance'' between \emph{label-sparse} samples with few (if any) segmentation labels and \emph{label-dense} samples with plentiful labeled pixels.
Existing distributionally robust algorithms have focused on adaptively truncating/down-weighting the ``less informative'' (\ie, label-sparse in our context) samples.  
To exploit data features of label-sparse samples more efficiently, we propose an adaptively weighted online optimization algorithm --- \ours --- to incorporate data augmentation consistency regularization in sample reweighting. Our method introduces a set of trainable weights to balance the supervised loss and unsupervised consistency regularization of each sample separately. At the saddle point of the underlying objective, the weights assign label-dense samples to the supervised loss and label-sparse samples to the unsupervised consistency regularization.
We provide a convergence guarantee by recasting the optimization as online mirror descent on a saddle point problem. Our empirical results demonstrate that \ours not only enhances the segmentation performance and sample efficiency but also improves the robustness to concept shift on various medical image segmentation tasks with different UNet-style backbones.
\end{abstract}
\section{Introduction}
Modern machine learning is revolutionizing the field of medical imaging, especially in computer-aided diagnosis with computed tomography (CT) and magnetic resonance imaging (MRI) scans. 
However, classical learning objectives like empirical risk minimization (ERM) generally assume that training samples are independently and identically (\iid) distributed, whereas real-world medical image data rarely satisfy this assumption. 
\Cref{fig:synapse_case40_sparse_dense_ce_weights} instantiates a common observation in medical image segmentation where the segmentation labels corresponding to different cross-sections of the human body tend to have distinct proportions of labeled (\ie, non-background) pixels, which is accurately reflected by the evaluation of supervised cross-entropy loss during training. 
We refer to this as the ``information imbalance'' among samples, as opposed to the well-studied ``class imbalance''~\citep{wong2018segmentation,taghanaki2019combo,yeung2022unified} among the numbers of segmentation labels in different classes.
Such information imbalance induces distinct difficulty/paces of learning with the cross-entropy loss for different samples~\citep{wang2021survey, tullis2011effectiveness,tang2018attention,hacohen2019power}.
Specifically, we say a sample is \emph{label-sparse} when it contains very few (if any) segmentation labels; in contrast, a sample is \emph{label-dense} when its segmentation labels are prolific. Motivated by the information imbalance among samples, we explore the following questions:
\begin{center}
    \textit{
        What is the effect of separation between sparse and dense labels on segmentation?
        \\ 
        Can we leverage such information imbalance to improve the segmentation accuracy?
    }   
\end{center}

We formulate the mixture of label-sparse and label-dense samples as a concept shift --- a type of distribution shift in the conditional distribution of labels given features $P\rsep{\yb}{\xb}$. 
Coping with concept shifts, prior works have focused on adaptively truncating (hard-thresholding) the empirical loss associated with label-sparse samples.  These include the Trimmed Loss Estimator~\citep{shen2019learning}, MKL-SGD~\citep{pmlr-v108-shah20a}, Ordered SGD~\citep{kawaguchi2020ordered}, and the quantile-based Kacmarz algorithm~\citep{haddock_quantile-based_2020}. 
Alternatively, another line of works~\citep{wang_iterative_2018,sagawa2019distributionally} proposes to relax the hard-thresholding operation to soft-thresholding by down-weighting instead of truncating the less informative samples.  
However, diminishing sample weights reduces the importance of both the features and the labels simultaneously, which is still not ideal as the potentially valuable information in the features of the label-sparse samples may not be fully used.

\begin{figure}[ht]
    \centering
    \includegraphics[width=.8\linewidth]{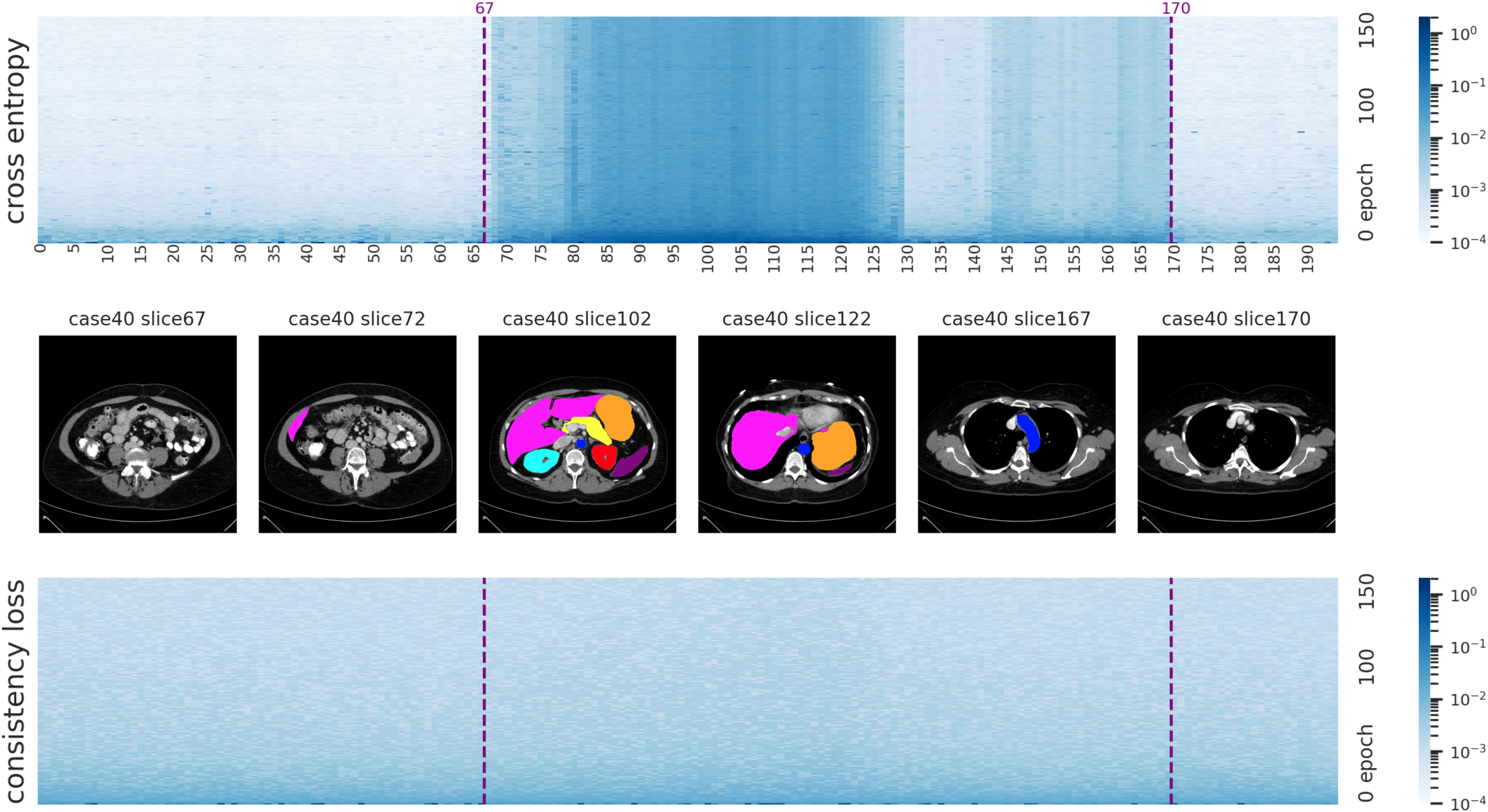}
    \caption{Evolution of cross-entropy losses versus consistency regularization terms for slices at different cross-sections of the human body in the Synapse dataset (described in \Cref{sec:experiments}) during training.}
    \label{fig:synapse_case40_sparse_dense_ce_weights}
\end{figure}

For further exploitation of the feature of training samples, we propose the incorporation of \emph{data augmentation consistency regularization} on label-sparse samples. 
As a prevalent strategy for utilizing unlabeled data, consistency regularization~\citep{bachman2014learning,laine2016temporal,sohn2020fixmatch} encourages data augmentations of the same samples to lie in the vicinity of each other on a proper manifold.  
For medical imaging segmentation, consistency regularization has been extensively studied in the semi-supervised learning setting~\citep{bortsova2019semi,zhao2019data,li2020transformation,wang2021deep,zhang2021multi,zhou2021ssmd,basak2022embarrassingly} as a strategy for overcoming label scarcity.
Nevertheless, unlike general vision tasks, for medical image segmentation, the scantiness of unlabeled image data can also be a problem due to regulations and privacy considerations~\cite{karimi2020deep}, which makes it worthwhile to reminisce the more classical supervised learning setting. 
In contrast to the aforementioned semi-supervised strategies, we  explore the potency of consistency regularization in the \emph{supervised learning} setting by leveraging the information in the features of label-sparse samples via data augmentation consistency regularization.

To naturally distinguish the label-sparse and label-dense samples, we make a key observation that the unsupervised consistency regularization on encoder layer outputs (of a UNet-style architecture) is much more uniform across different subpopulations than the supervised cross-entropy loss (as exemplified in \Cref{fig:synapse_case40_sparse_dense_ce_weights}).
Since the consistency regularization is characterized by the marginal distribution of features $P(\xb)$ but not labels and therefore is less affected by the concept shift in $P\rsep{\yb}{\xb}$, it serves as a natural reference for separating the label-sparse and label-dense samples. 
In light of this observation, we present the \emph{weighted data augmentation consistency (WAC) regularization} --- a minimax formulation that reweights the cross-entropy loss versus the consistency regularization associated with each sample via a set of trainable weights. 
At the saddle point of this minimax formulation, the WAC regularization automatically separates samples from different subpopulations by assigning all weights to the consistency regularization for label-sparse samples, and all weights to the cross-entropy terms for label-dense samples.

We further introduce an adaptively weighted online optimization algorithm, \ours, for solving the minimax problem posed by the WAC regularization, which is inspired by a mirror-descent-based algorithm for distributionally robust optimization~\citep{sagawa2019distributionally}. By adaptively learning the weights between the cross-entropy loss and consistency regularization of different samples, \ours comes with both a convergence guarantee and empirical success.

The main contributions are summarized as follows:
\begin{itemize}
    \item We introduce the \emph{WAC regularization} that leverages the consistency regularization on the encoder layer outputs (of a UNet-style architecture) as a natural reference to distinguish the label-sparse and label-dense samples (\Cref{sec:wac}). 
    \item We propose an adaptively weighted online optimization algorithm --- \ours --- for solving the WAC regularization problem with a convergence guarantee (\Cref{sec:ada_wac}). 
    \item Through extensive experiments on different medical image segmentation tasks with different UNet-style backbone architectures, we demonstrate the effectiveness of \ours not only for enhancing the segmentation performance and sample efficiency but also for improving the robustness to concept shift (\Cref{sec:experiments}).
\end{itemize}

\subsection{Related Work}

\paragraph{Sample reweighting.} 
Sample reweighting is a popular strategy for dealing with distribution/subpopulation shifts in training data where different weights are assigned to samples from different subpopulations. 
In particular, the distributionally-robust optimization (DRO) framework~\citep{bental2013robust,duchi2016statistics,duchi2018learning,sagawa2019distributionally} considers a collection of training sample groups from different distributions.  With the explicit grouping of samples, the goal is to minimize the worst-case loss over the groups.
Without prior knowledge of sample grouping, importance sampling~\citep{needell2014stochastic,zhao2015stochastic,alain2015variance,loshchilov2015online,gopal2016adaptive,katharopoulos2018not}, iterative trimming~\citep{kawaguchi2020ordered,shen2019learning}, and empirical-loss-based reweighting~\citep{Wu_Xie_Du_Ward_2022} are commonly incorporated in the stochastic optimization process for adaptive reweighting and separation of samples from different subpopulations.

\paragraph{Data augmentation consistency regularization.}
As a popular way of exploiting data augmentations, consistency regularization encourages models to learn the vicinity among augmentations of the same sample based on the assumption that data augmentations generally preserve the semantic information in data and therefore lie closely on proper manifolds. Beyond being a powerful building block in semi-supervised~\citep{bachman2014learning,sajjadi2016regularization,laine2016temporal,sohn2020fixmatch,berthelot2019mixmatch} and self-supervised~\citep{wu2018unsupervised,he2020momentum,chen2020simple,grill2020bootstrap} learning, the incorporation of data augmentation and consistency regularization also provably improves generalization and feature learning even in the supervised learning setting~\citep{yang2022sample,shen2022data}.

For medical imaging, data augmentation consistency regularization is generally leveraged as a semi-supervised learning tool~\citep{bortsova2019semi,zhao2019data,li2020transformation,wang2021deep,zhang2021multi,zhou2021ssmd,basak2022embarrassingly}. In efforts to incorporate consistency regularization in segmentation tasks with augmentation-sensitive labels, \cite{li2020transformation} encourages transformation consistency between predictions with augmentations applied to the image inputs and the segmentation outputs.
\cite{basak2022embarrassingly} penalizes inconsistent segmentation outputs between teacher-student models, with MixUp \citep{zhang2017mixup} applied to image inputs of the teacher model and segmentation outputs of the student model.
Instead of enforcing consistency in the segmentation output space as above, we leverage the insensitivity of sparse labels to augmentations and encourage consistent encodings (in the latent space of encoder outputs) on label-sparse samples.
\section{Problem Setup}\label{sec:problem_setup}

\paragraph{Notation.} 
For any $K \in \N$, we denote $\sbr{K} = \cbr{1,\dots,K}$. 
We represent the elements and subtensors of an arbitrary tensor by adapting the syntax for Python slicing on the subscript (except counting from $1$). For example, $\iloc{\xb}{i,j}$ denotes the $(i,j)$-entry of the two-dimensional tensor $\xb$, and $\iloc{\xb}{i,:}$ denotes the $i$-th row.
Let $\II$ be a function onto $\cbr{0,1}$ such that, for any event $e$, $\iffun{e}=1$ if $e$ is true and $0$ otherwise.
For any distribution $P$ and $n \in \N$, we let $P^n$ denote the joint distribution of $n$ samples drawn \iid from $P$.
Finally, we say that an event happens with high probability (\whp) if the event takes place with probability $1-\Omega\rbr{\poly\rbr{n}}^{-1}$.

\subsection{Pixel-wise Classification with Sparse and Dense Labels}\label{subsec:pixel-wise-classification}

We consider medical image segmentation as a pixel-wise multi-class classification problem where we aim to learn a pixel-wise classifier $h:\Xcal \to [K]^d$ that serves as a good approximation to the ground truth $h^*: \Xcal \to [K]^d$. 

Recall the separation of cross-entropy losses between samples with different proportions of background pixels from \Cref{fig:synapse_case40_sparse_dense_ce_weights}.
We refer to a sample $\rbr{\xb,\yb} \in \Xcal \times [K]^d$ as \emph{label-sparse} if most pixels in $\yb$ are labeled as background; for these samples, the cross-entropy loss on $\rbr{\xb,\yb}$ converges rapidly in the early stage of training.
Otherwise, we say that $\rbr{\xb,\yb}$ is \emph{label-dense}.
Formally, we describe such variation as a concept shift in the data distribution.
\begin{definition}[Mixture of label-sparse and label-dense subpopulations]
We assume that \emph{label-sparse and label-dense samples} are drawn from $P_0$ and $P_1$ with distinct conditional distributions $P_0\rsep{\yb}{\xb}$ and $P_1\rsep{\yb}{\xb}$ but common marginal distribution $P\rbr{\xb}$ such that $P_i\rbr{\xb,\yb} = P_i\rsep{\yb}{\xb} P\rbr{\xb}$ ($i=0,1$).
For $\xi \in [0,1]$, we define $P_\xi$ as a data distribution where $\rbr{\xb,\yb} \sim P_\xi$ is drawn either from $P_1$ with probability $\xi$ or from $P_0$ with probability $1-\xi$.
\end{definition}

We aim to learn a pixel-wise classifier from a function class $\Hcal$ where every $h_\theta \in \Hcal$ satisfies $\iloc{h_\theta\rbr{\xb}}{j} = \argmax_{k \in [K]} \iloc{f_\theta\rbr{\xb}}{j,:}$ for all $j \in [d]$, and the underlying function $f_\theta \in \Fcal$, parameterized by some $\theta \in \Fcal_\theta$, admits an encoder-decoder structure:
\begin{align}\label{eq:encoder_decoder}
    \Fcal \subseteq \csepp{f_\theta = \phi_\theta \circ \psi_\theta}{\phi_\theta: \Xcal \to \Zcal, \psi_\theta:\Zcal \to [0,1]^{d \times K}}.
\end{align}
Here $\phi_\theta, \psi_\theta$ correspond to the encoder and decoder functions, respectively. The parameter space $\Fcal_\theta$ is equipped with the norm $\nbr{\cdot}_\Fcal$ and its dual norm $\nbr{\cdot}_{\Fcal,*}$\footnote{
For \ours (\Cref{prop:convergence_srw_dac} in \Cref{sec:ada_wac}), $\Fcal_\theta$ is simply a subspace in the Euclidean space with dimension equal to the total number of parameters for each $\theta \in \Fcal_\theta$, with $\nbr{\cdot}_\Fcal$ and $\nbr{\cdot}_{\Fcal,*}$ both being the $\ell_2$-norm.
}. $\rbr{\Zcal, \varrho}$ is a latent metric space. 

To learn from segmentation labels, we consider the \emph{averaged cross-entropy loss}:
\begin{align}\label{eq:cross_entropy_loss}
\begin{split}
    \lossce\rbr{\theta;(\xb,\yb)} 
    = &-\frac{1}{d} \sum_{j=1}^d \sum_{k=1}^K \iffun{\iloc{\yb}{j}=k} \cdot \log\rbr{ \iloc{f_\theta\rbr{\xb}}{j,k} } \\
    = &-\frac{1}{d} \sum_{j=1}^d \log\rbr{\iloc{f_\theta\rbr{\xb}}{j, \iloc{\yb}{j}}}.
\end{split}
\end{align} 
We assume proper learning with $\theta^* \in \bigcap_{\xi \in [0,1]} \argmin_{\theta \in \Fcal_\theta} \E_{\rbr{\xb,\yb}\sim P_\xi}\sbr{\lossce\rbr{\theta;(\xb,\yb)} }$ being invariant with respect to $\xi$.\footnote{
We assume proper learning only to \begin{enumerate*}[label=(\roman*)]
    \item highlight the invariance of the desired ground truth to $\xi$ that can be challenging to learn with finite samples in practice and
    \item provide a natural pivot for the convex and compact neighborhood $\Fnb{\gamma}$ of ground truth $\theta^*$ in \Cref{ass:separability_sparse_dense} granted by the pretrained initialization, where $\theta^*$ can also be replaced with the pretrained initialization weights $\theta_0 \in \Fnb{\gamma}$.
    In particular, neither our theory nor the \ours algorithm requires the function class $\Fcal$ to be expressive enough to truly contain such $\theta^*$.
\end{enumerate*}
}

\subsection{Augmentation Consistency Regularization}
Despite the invariance of $f_{\theta^*}$ to $P_\xi$ on the population loss, with a finite number of training samples in practice, the predominance of label-sparse samples would be problematic. As an extreme scenario for the pixel-wise classifier with encoder-decoder structure (\Cref{eq:encoder_decoder}), when the label-sparse samples are predominant ($\xi \ll 1$), a decoder function $\psi_\theta$ that predicts every pixel as background can achieve near-optimal cross-entropy loss, regardless of the encoder function $\phi_\theta$, considerably compromising the test performance (\cf \Cref{tab:synapse_sample_eff}). 
To encourage legit encoding even in the absence of sufficient dense labels, we leverage the unsupervised consistency regularization on the \emph{encoder function} $\phi_\theta$ based on data augmentations. 

Let $\Acal$ be a distribution over transformations on $\Xcal$ where for any $\xb \in \Xcal$, each $A \sim \Acal$ ($A:\Xcal \to \Xcal$) induces an augmentation $A\rbr{\xb}$ of $\xb$ that perturbs low-level information in $\xb$. 
We aim to learn an encoder function $\phi_\theta:\Xcal \to \Zcal$ that is capable of filtering out low-level information from $\xb$ and therefore provides similar encodings for augmentations of the same sample. 
Recalling the metric $\varrho$ (\eg, the Euclidean distance) on $\Zcal$, for a given scaling hyperparameter $\lambdac>0$, we measure the similarity between augmentations with a consistency regularization term on $\phi_\theta\rbr{\cdot}$: for any $A_1,A_2 \sim \Acal^2$,
\begin{align}\label{eq:dac}
    \regdac\rbr{\theta;\xb,A_1,A_2} \dfeq \lambdac \cdot \varrho\Big( \phi_\theta\rbr{A_1(\xb)}, \phi_\theta\rbr{A_2(\xb)} \Big).
\end{align}

For the $n$ training samples $\cbr{\rbr{\xb_i, \yb_i}}_{i \in [n]}\sim P_\xi^n$, we consider $n$ pairs of data augmentation transformations $\cbr{\rbr{A_{i,1}, A_{i,2}}}_{i \in [n]} \sim \Acal^{2n}$.
In the basic version, we encourage the similar encoding $\phi_\theta\rbr{\cdot}$ of the augmentation pairs $\rbr{{A_{i,1}\rbr{\xb_i}}, {A_{i,2}\rbr{\xb_i}}}$ for all $i \in [n]$ via consistency regularization:
\begin{align}\label{eq:objective_dac_encoder_plain}
    \min_{\theta \in \Fnb{\gamma}} \frac{1}{n} \sum_{i=1}^n 
    \lossce\rbr{\theta;\rbr{\xb_i,\yb_i}} + 
    \regdac\rbr{\theta;\xb_i,A_{i,1},A_{i,2}}.
\end{align}

We enforce consistency on $\phi_\theta\rbr{\cdot}$ in light of the encoder-decoder architecture: the encoder is generally designed to abstract essential information and filters out low-level non-semantic perturbations (\eg, those introduced by augmentations), while the decoder recovers the low-level information for the pixel-wise classification.
Therefore, with different $A_1, A_2 \sim \Acal$, the encoder output $\phi_\theta\rbr{\cdot}$ tends to be more consistent than the other intermediate layers, especially for label-dense samples.

\section{Weighted Data Augmentation Consistency (WAC) Regularization}\label{sec:wac}

As the motivation, we begin with a key observation about the averaged cross-entropy:  
\begin{remark}[Separation of averaged cross-entropy loss on $P_0$ and $P_1$]\label{rmk:separation_average_cross_entropy}
    As demonstrated in \Cref{fig:synapse_case40_sparse_dense_ce_weights}, the sparse labels from $P_0$ tend to be much easier to learn than the dense ones from $P_1$, leading to considerable separation of averaged cross-entropy losses on the sparse and dense labels after a sufficient number of training epochs.  In other words, $\lossce\rbr{\theta;\rbr{\xb,\yb}} \ll \lossce\rbr{\theta;\rbr{\xb',\yb'}}$ for label-sparse samples $\rbr{\xb,\yb} \sim P_0$ and label-dense samples $\rbr{\xb',\yb'} \sim P_1$ with high probability.
\end{remark}

Although \Cref{eq:objective_dac_encoder_plain} with consistency regularization alone can boost the segmentation accuracy during testing (\cf \Cref{tab:ablation}), it does not take the separation between label-sparse and label-dense samples into account. In \Cref{sec:experiments}, we will empirically demonstrate that proper exploitation of such separation, like the formulation introduced below, can lead to improved classification performance.

We formalize the notion of separation between $P_0$ and $P_1$ with the consistency regularization (\Cref{eq:dac}) as a reference in the following assumption
\footnote{
Although \Cref{ass:separability_sparse_dense} may seem to be rather strong, it is only required for the separation guarantee of label-sparse and label-dense samples with high probability in \Cref{prop:spontaneous_separation_sparse_dense}, but not for the adaptive weighting algorithm introduced in \Cref{sec:ada_wac} or in practice for the experiments.
}.

\begin{assumption}[$n$-separation between $P_0$ and $P_1$]\label{ass:separability_sparse_dense}
    Given a sufficiently small $\gamma>0$, let $\Fnb{\gamma} = \csepp{\theta \in \Fcal_\theta}{\nbr{\theta - \theta^*}_{\Fcal} \le \gamma}$ be a compact and convex neighborhood of well-trained pixel-wise classifiers\footnote{With pretrained initialization, we assume that the optimization algorithm is always probing in $\Fnb{\gamma}$.}. We say that \emph{$P_0$ and $P_1$ are $n$-separated over $\Fnb{\gamma}$} if there exists $\omega>0$ such that with probability $1-\Omega\rbr{n^{1+\omega}}^{-1}$ over $\rbr{\rbr{\xb,\yb},\rbr{A_1,A_2}} \sim P_\xi \times \Acal^2$, the following hold:
    \begin{enumerate}[label=(\roman*)]
        \item $\lossce\rbr{\theta;\rbr{\xb,\yb}} < \regdac\rbr{\theta;\xb,A_1,A_2}$ for all $\theta \in \Fnb{\gamma}$ given $\rbr{\xb,\yb} \sim P_0$;
        \item $\lossce\rbr{\theta;\rbr{\xb,\yb}} > \regdac\rbr{\theta;\xb,A_1,A_2}$ for all $\theta \in \Fnb{\gamma}$ given $\rbr{\xb,\yb} \sim P_1$.
    \end{enumerate}
\end{assumption}
This assumption is motivated by the empirical observation that the perturbation in $\phi_\theta\rbr{\cdot}$ induced by $\Acal$ is more uniform across $P_0$ and $P_1$ than the averaged cross-entropy losses, as instantiated in \Cref{fig:synapse-weights}. 

Under \Cref{ass:separability_sparse_dense}, up to a proper scaling hyperparameter $\lambdac$, the consistency regularization (\Cref{eq:dac}) can separate the averaged cross-entropy loss (\Cref{eq:cross_entropy_loss}) on $n$ label-sparse and label-dense samples with probability $1-\Omega\rbr{n^{\omega}}^{-1}$ (as explained formally in \Cref{subapx:spontaneous_separation_sparse_dense}). 
In particular, the larger $n$ corresponds to the stronger separation between $P_0$ and $P_1$. 

With \Cref{ass:separability_sparse_dense}, we introduce a minimax formulation that incentivizes the separation of label-sparse and label-dense samples automatically 
by introducing a flexible weight $\iloc{\betab}{i} \in [0,1]$ that balances $\lossce\rbr{\theta;\rbr{\xb_i,\yb_i}}$ and $\regdac\rbr{\theta;\xb_i,A_{i,1},A_{i,2}}$ for each of the $n$ samples.
\begin{align}\label{eq:objective_dac_encoder_weighted}
\begin{split}
    &\wh\theta^\wdac, \wh\betab \in 
    \argmin_{\theta \in \Fnb{\gamma}}~
    \argmax_{\betab \in [0,1]^n}~ 
    \wh L^\wdac\rbr{\theta,\betab} 
    \\
    &\wh L^\wdac\rbr{\theta,\betab} \dfeq \frac{1}{n} \sum_{i=1}^n \iloc{\betab}{i} \cdot \lossce\rbr{\theta;\rbr{\xb_i,\yb_i}} + (1-\iloc{\betab}{i}) \cdot \regdac\rbr{\theta;\xb_i,A_{i,1},A_{i,2}}. 
\end{split}
\end{align}

With convex and continuous loss and regularization terms (formally in \Cref{prop:spontaneous_separation_sparse_dense}), \Cref{eq:objective_dac_encoder_weighted} admits a saddle point corresponding to $\wh\betab$ which separates the label-sparse and label-dense samples under \Cref{ass:separability_sparse_dense}.
\begin{proposition}[Formal proof in \Cref{subapx:spontaneous_separation_sparse_dense}]\label{prop:spontaneous_separation_sparse_dense}
    Assume the convexity and continuity in $\theta$ of both
    $\lossce\rbr{\theta;\rbr{\xb,\yb}}$ and $\regdac\rbr{\theta;\xb,A_{1},A_{2}}$ for all $(\xb,\yb) \in \Xcal \times [K]^d$ and $A_1,A_2 \sim \Acal^2$; $\Fnb{\gamma} \subset \Fcal_\theta$ is compact and convex.
    If $P_0$ and $P_1$ are $n$-separated (\Cref{ass:separability_sparse_dense}), then there exists $\wh\betab \in \cbr{0,1}^n$ and $\wh\theta^\wdac \in \argmin_{\theta \in \Fnb{\gamma}} \wh L^\wdac\rbr{\theta, \wh\betab}$ such that 
    \begin{align}\label{eq:saddle_point_def}
    \begin{split}
        \min_{\theta \in \Fnb{\gamma}} \wh L^\wdac\rbr{\theta, \wh\betab} 
        = \wh L^\wdac\rbr{\wh\theta^\wdac, \wh\betab} 
        = \max_{\betab \in [0,1]^n} \wh L^\wdac\rbr{\wh\theta^\wdac, \betab}.
    \end{split}
    \end{align}
    Further, $\wh\betab$ separates the label-sparse and label-dense samples --- $\iloc{\wh\betab}{i}=\iffun{\rbr{\xb_i,\yb_i} \sim P_1}$ --- \whp. 
\end{proposition}
That is, for $n$ samples drawn from a mixture of $n$-separated $P_0$ and $P_1$, the saddle point of $L^\wdac_i \rbr{\theta,\betab}$ in \Cref{eq:objective_dac_encoder_weighted} corresponds to $\iloc{\betab}{i} = 0$ on label-sparse samples (\ie, learning from the unsupervised consistency regularization), and $\iloc{\betab}{i} = 1$ on label-dense samples (\ie, learning from the supervised averaged cross-entropy loss).

\begin{remark}[Connection to hard-thresholding algorithms]\label{rmk:relation_hard_thresholding}
    The saddle point of \Cref{eq:objective_dac_encoder_weighted} is closely related to hard-thresholding algorithms like Ordered SGD~\citep{kawaguchi2020ordered} and iterative trimmed loss~\citep{shen2019learning}. 
    In each iteration, these algorithms update the model only on a proper subset of training samples based on the (ranking of) current empirical risks. 
    Compared to hard-thresholding algorithms,  
    \begin{enumerate*}[label=(\roman*)]
        \item \Cref{eq:objective_dac_encoder_weighted} additionally leverages the unused samples (\eg, label-sparse samples) for unsupervised consistency regularization on data augmentations;
        \item meanwhile, it does not require prior knowledge of the sample subpopulations (\eg, $\xi$ for $P_\xi$) which is essential for hard-thresholding algorithms. 
        % {\color{magenta} This is a good remark!} 
    \end{enumerate*}

    \Cref{eq:objective_dac_encoder_weighted} further facilitates the more flexible optimization process. As we will empirically show in \Cref{tab:trim}, despite the close relation between \Cref{eq:objective_dac_encoder_weighted} and the hard-thresholding algorithms (\Cref{rmk:relation_hard_thresholding}), such updating strategies may be suboptimal for solving \Cref{eq:objective_dac_encoder_weighted}.
\end{remark}

\section{Adaptively Weighted Data Augmentation Consistency Regularization (\ours)}\label{sec:ada_wac}

Inspired by the breakthrough made by \cite{sagawa2019distributionally} in the distributionally-robust optimization (DRO) setting where gradient updating on weights is shown to enjoy better convergence guarantees than hard thresholding, we introduce an adaptively weighted online optimization algorithm (\Cref{alg:srw_dac}) for solving \Cref{eq:objective_dac_encoder_weighted} based on online mirror descent.

In contrast to the commonly used stochastic gradient descent (SGD), the flexibility of online mirror descent in choosing the associated norm space not only allows gradient updates on sample weights but also grants distinct learning dynamics to sample weights $\betab_t$ and model parameters $\theta_t$, which leads to the following convergence guarantee.
\begin{proposition}[Formally in \Cref{prop:convergence_srw_dac_formal}, proof in \Cref{subapx:convergence_wdac}, assumptions instantiated in \Cref{ex:convex_continuous}]\label{prop:convergence_srw_dac}
    Assume that 
    $\lossce\rbr{\theta;\rbr{\xb,\yb}}$ and $\regdac\rbr{\theta;\xb,A_{1},A_{2}}$ are convex and continuous in $\theta$ for all $(\xb,\yb) \in \Xcal \times [K]^d$ and $A_1,A_2 \sim \Acal^2$.  Assume moreover that  $\Fnb{\gamma} \subset \Fcal_\theta$ is convex and compact. If there exist
    \footnote{
        Following the convention, we use $*$ in subscription to denote the dual spaces. For instance, recalling the parameter space $\Fcal_\theta$ characterized by the norm $\nbr{\cdot}_{\Fcal}$ from \Cref{subsec:pixel-wise-classification}, we use $\nbr{\cdot}_{\Fcal,*}$ to denote its dual norm; while $C_{\theta,*}, C_{\betab,*}$ upper bound the dual norms of the gradients with respect to $\theta$ and $\betab$.
    } 
    $C_{\theta,*} > 0$ and $C_{\betab,*} > 0$ such that 
    \begin{align*}
        &\frac{1}{n} \sum_{i=1}^n \nbr{\nabla_\theta \wh L_i^\wdac\rbr{\theta,\betab}}_{\Fcal,*}^2 \le C_{\theta,*}^2 \\
        &\frac{1}{n} \sum_{i=1}^n \max\cbr{\lossce\rbr{\theta;\rbr{\xb_i,\yb_i}}, \regdac\rbr{\theta;\xb_i,A_{i,1},A_{i,2}}}^2 
        \le C_{\betab,*}^2
    \end{align*}
    for all $\theta \in \Fnb{\gamma}$, $\betab \in [0,1]^n$,
    then with $\eta_\theta = \eta_{\betab} = \frac{2}{\sqrt{5T \rbr{\gamma^2 C_{\theta,*}^2 + 2 n C_{\betab,*}^2}}}$, \Cref{alg:srw_dac} provides
    \begin{align*}
        \E\sbr{\max_{\betab \in [0,1]^n} \wh L^\wdac\rbr{\overline{\theta}_T, \betab} - \min_{\theta \in \Fnb{\gamma}} \wh L^\wdac\rbr{\theta, \overline{\betab}_T}}
        \le 2 \sqrt{\frac{5 \rbr{\gamma^2 C_{\theta,*}^2 + 2 n C_{\betab,*}^2}}{T}}
    \end{align*}
    where $\overline{\theta}_T = \frac{1}{T} \sum_{t=1}^T \theta_t$ and $\overline{\betab}_T = \frac{1}{T} \sum_{t=1}^T \betab_t$. 
\end{proposition} 

\begin{algorithm}[h]
\caption{\ours}\label{alg:srw_dac}
\begin{algorithmic}
    \STATE {\bfseries Input:} 
    Training samples $\cbr{\rbr{\xb_i, \yb_i}}_{i \in [n]}\sim P_\xi^n$, 
    augmentations $\cbr{\rbr{A_{i,1}, A_{i,2}}}_{i \in [n]} \sim \Acal^{2n}$,
    maximum number of iterations $T \in \N$, 
    learning rates $\eta_\theta, \eta_{\betab} > 0$, 
    pretrained initialization for the pixel-wise classifier $\theta_0 \in \Fnb{\gamma}$.
    
    \STATE Initialize the sample weights $\betab_0 = \b{1}/2 \in [0,1]^n$.
    
    \FOR{$t = 1,\dots,T$}
        \STATE Sample $i_t \sim [n]$ uniformly
        \STATE $\bb \gets \sbr{\iloc{\rbr{\betab_{t-1}}}{i_t}, 1-\iloc{\rbr{\betab_{t-1}}}{i_t}}$
        \STATE $\iloc{\bb}{1} \gets \iloc{\bb}{1} \cdot \exp\rbr{\eta_{\betab} \cdot \lossce\rbr{\theta_{t-1};\rbr{\xb_{i_t},\yb_{i_t}}}}$
        \STATE $\iloc{\bb}{2} \gets \iloc{\bb}{2} \cdot \exp\rbr{ \eta_{\betab} \cdot \regdac\rbr{\theta_{t-1};\xb_{i_t},A_{i_t,1},A_{i_t,2}} }$
        \STATE $\betab_t \gets \betab_{t-1}$, $\iloc{\rbr{\betab_t}}{i_t} \gets \iloc{\bb}{1}/\nbr{\bb}_1$
        \STATE $\theta_t \gets \theta_{t-1} - \eta_\theta \cdot \Big(
        \iloc{\rbr{\betab_t}}{i_t} \cdot \nabla_\theta\lossce\rbr{\theta_{t-1};\rbr{\xb_{i_t},\yb_{i_t}}} + \rbr{1-\iloc{\rbr{\betab_t}}{i_t}} \cdot \nabla_\theta\regdac\rbr{\theta_{t-1};\xb_{i_t},A_{i_t,1},A_{i_t,2}} \Big)$
    \ENDFOR
\end{algorithmic}
\end{algorithm}

In addition to the convergence guarantee, \Cref{alg:srw_dac} also demonstrates superior performance over hard-thresholding algorithms for segmentation problems in practice (\Cref{tab:trim}).
An intuitive explanation is that instead of filtering out all the label-sparse samples via hard thresholding, the adaptive weighting allows the model to learn from some sparse labels at the early epochs, while smoothly down-weighting $\lossce$ of these samples since learning sparse labels tends to be easier (\Cref{rmk:separation_average_cross_entropy}). 
With the learned model tested on a mixture of label-sparse and label-dense samples, learning sparse labels at the early stage is crucial for accurate segmentation.

\section{Experiments}\label{sec:experiments}
In this section, we investigate the proposed \ours algorithm (\Cref{alg:srw_dac}) on different medical image segmentation tasks with different UNet-style architectures. We first demonstrate the performance improvements brought by \ours in terms of sample efficiency and robustness to concept shift (\Cref{tab:synapse_sample_eff}). Then, we verify the empirical advantage of \ours compared to the closely related hard-thresholding algorithms as discussed in \Cref{rmk:relation_hard_thresholding} (\Cref{tab:trim}). Our ablation study (\Cref{tab:ablation}) further illustrates the indispensability of both sample reweighting and consistency regularization, the deliberate combination of which leads to the superior performance of \ours\footnote{We release our code at \href{https://github.com/gail-yxie/adawac}{https://github.com/gail-yxie/adawac}.}.

\paragraph{Experiment setup.} 
We conduct experiments on two medical image segmentation tasks: abdominal CT segmentation for Synapse multi-organ dataset (Synapse)\footnote{\href{https://www.synapse.org/\#!Synapse:syn3193805/wiki/217789}{{https://www.synapse.org/\#!Synapse:syn3193805/wiki/217789}}} and cine-MRI segmentation for Automated cardiac diagnosis challenge dataset (ACDC)\footnote{\href{https://www.creatis.insa-lyon.fr/Challenge/acdc/}{https://www.creatis.insa-lyon.fr/Challenge/acdc/}}, with two UNet-like architectures: TransUNet~\citep{chen2021transunet} and UNet~\cite{ronneberger2015unet} (deferred to \Cref{subsec:exp_unet}). For the main experiments with TransUNet in \Cref{sec:experiments}, we follow the official implementation in \citep{chen2021transunet} and use ERM+SGD as the baseline. We evaluate segmentations with two standard metrics---the average Dice-similarity coefficient (DSC) and the average $95$-percentile of Hausdorff distance (HD95). Dataset and implementation details are deferred to \Cref{app:imp}. Given the sensitivity of medical image semantics to perturbations, our experiments only involve simple augmentations (\ie, rotation and mirroring) adapted from \citep{chen2021transunet}.

It is worth highlighting that, in addition to the information imbalance among samples caused by the concept shift discussed in this work, the pixel-wise class imbalance (\eg, the predominance of background pixels) is another well-investigated challenge for medical image segmentation, where coupling the dice loss~\citep{wong2018segmentation,taghanaki2019combo,yeung2022unified} in the objective is a common remedy used in many state-of-the-art methods~\citep{chen2021transunet,cao2021swin}. The implementation of \ours also leverages the dice loss to alleviate pixel-wise class imbalance. We defer the detailed discussion to \Cref{apx:dice}.

\subsection{Segmentation Performance of \ours with TransUNet}\label{subsec:exp_transunet}
\paragraph{Segmentation on Synapse.}
\Cref{fig:synapse} visualizes the segmentation predictions on $6$ Synapse test slices given by models trained via \ours(ours) and via the baseline (ERM+SGD) with TransUNet~\citep{chen2021transunet}. We observe that \ours provides more accurate predictions on the segmentation boundaries and captures small organs better than the baseline.

\begin{figure}[ht]
    \centering
    \includegraphics[width=.75\linewidth]{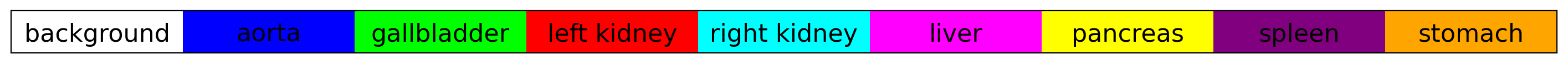}
    \includegraphics[width=.8\linewidth]{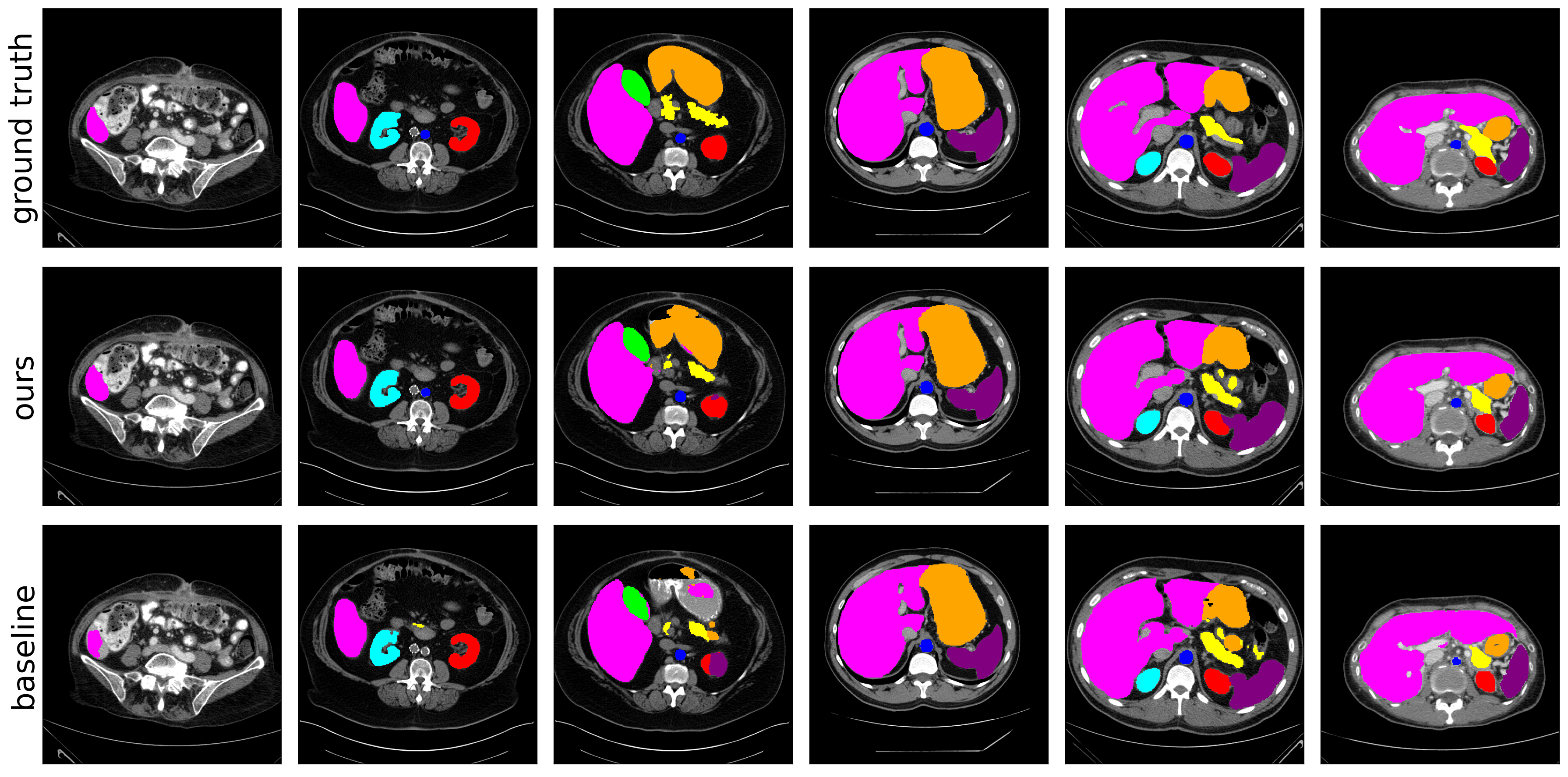}
    \caption{Visualization of segmentation predictions with TransUNet~\citep{chen2021transunet} on Synapse. Top to bottom: ground truth, ours (\ours), baseline.}
    \label{fig:synapse}
\end{figure}

\paragraph{Visualization of \ours.} 
As shown in \Cref{fig:synapse-weights}, with $\lossce\rbr{\theta_t;(\xb_i,\yb_i)}$ (\Cref{eq:cross_entropy_loss}) of label-sparse versus label-dense slices weakly separated in the early epochs, the model further learns to distinguish $\lossce\rbr{\theta_t;(\xb_i,\yb_i)}$ of label-sparse/label-dense slices during training. By contrast, $\regdac\rbr{\theta_t;\xb_i,A_{i,1},A_{i,2}}$ (\Cref{eq:dac}) remains mixed for all slices throughout the entire training process. As a result, the CE weights of label-sparse slices are much smaller than those of label-dense ones, pushing \ours to learn more image representations but fewer pixel classifications for slices with sparse labels and learn more pixel classifications for slices with dense labels.

\begin{figure}[ht]
    \centering
    \includegraphics[width=.8\linewidth]{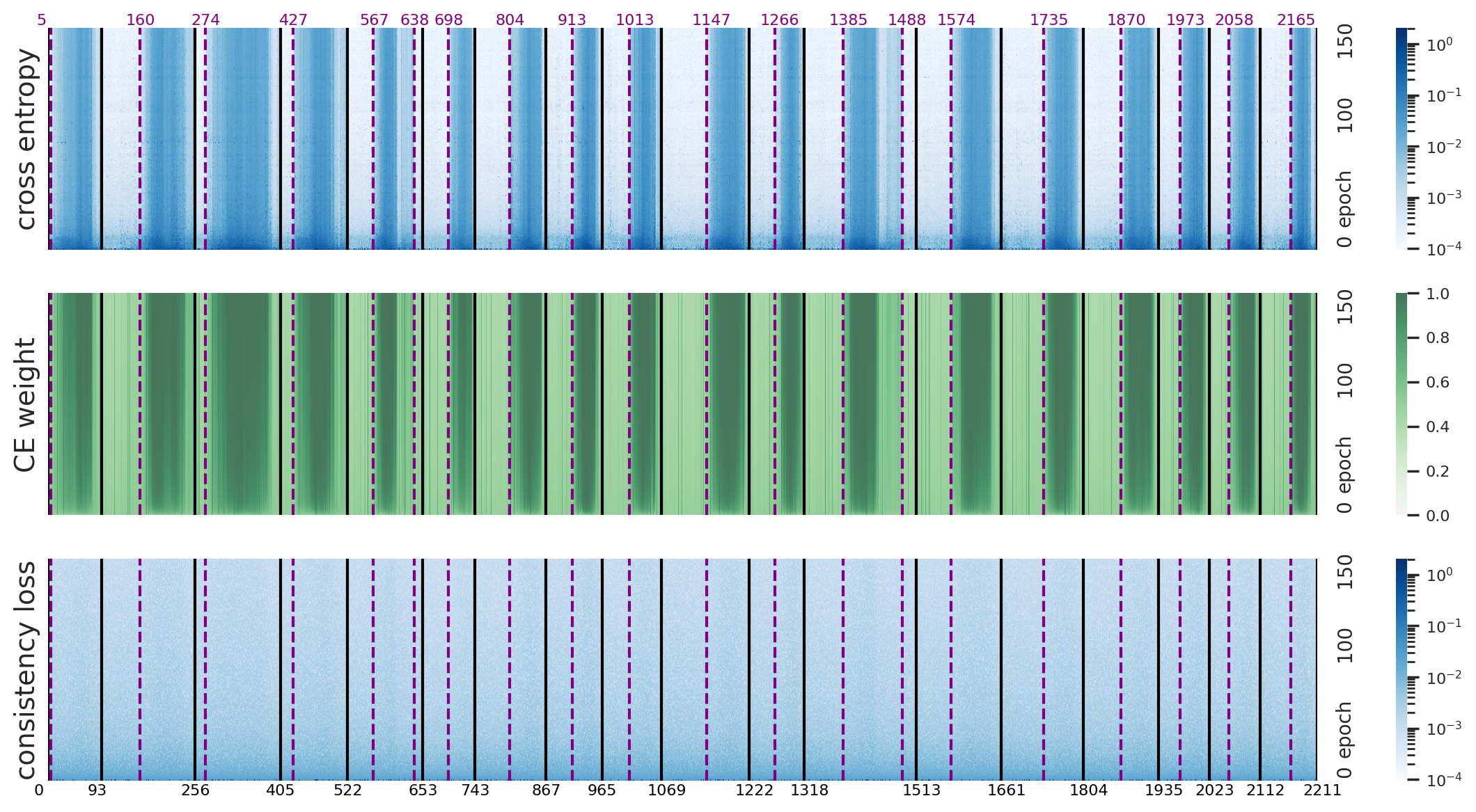}
    \caption{$\lossce\rbr{\theta_t;(\xb_i,\yb_i)}$ (top), CE weights $\betab_t$ (middle), and $\regdac\rbr{\theta_t;\xb_i,A_{i,1},A_{i,2}}$ (bottom) of the entire Synapse training process. The x-axis indexes slices 0--2211. The y-axis enumerates epochs 0--150. Individual cases (patients) are partitioned by black lines, while purple lines separate slices with/without non-background pixels.}
    \label{fig:synapse-weights}
\end{figure}

\paragraph{Sample efficiency and distributional robustness.} 
We first demonstrate the \emph{sample efficiency} of \ours in comparison to the baseline (ERM+SGD) when training only on different subsets of the full Synapse training set (``\b{full}'' in \Cref{tab:synapse_sample_eff}). Specifically, 
\begin{enumerate*}[label=(\roman*)]
    \item  \b{half-slice} contains slices with even indices only in each case (patient)\footnote{Such sampling is equivalent to doubling the time interval between two consecutive scans or halving the scanning frequency in practice, resulting in the halving of sample size.};
    \item \b{half-vol} consists of 9 cases uniformly sampled from the total 18 cases in \b{full} where different cases tend to have distinct $\xi$s (\ie, ratios of label-dense samples);
    \item \b{half-sparse} takes the first half slices in each case, most of which tend to be label-sparse (\ie, $\xi$s are made to be small).
\end{enumerate*}
As shown in \Cref{tab:synapse_sample_eff}, the model trained with \ours on \b{half-slice} generalizes as well as a baseline model trained on \b{full}, if not better.
Moreover, the \b{half-vol} and \b{half-sparse} experiments illustrate the \emph{robustness} of \ours to concept shift. 
Furthermore, such sample efficiency and distributional robustness of \ours extend to the more widely used UNet architecture. We defer the detailed results and discussions on UNet to \Cref{subsec:exp_unet}.

\begin{table*}[!ht]
    \caption{\ours with TransUNet trained on the full Synapse and its subsets. 
    }\vspace{-.5em}
    \label{tab:synapse_sample_eff}
    \centering
    \begin{adjustbox}{width=1\textwidth}  
    \begin{tabular}{ll|cc|cccccccc}
    \toprule
    Training & Method &  DSC $\uparrow$ & HD95 $\downarrow$ & Aorta & Gallbladder & Kidney (L) & Kidney (R) & Liver & Pancreas & Spleen & Stomach 
    \\
    \midrule
    \multirow{2}{*}{full} 
    & baseline & 76.66 $\pm$ 0.88 & 29.23 $\pm$ 1.90 & 87.06 & 55.90 & 	81.95 & 75.58 & 94.29 & 56.30 & 86.05 & 76.17
    \\
    & \ours & \best{79.04 $\pm$ 0.21} & \best{27.39 $\pm$ 1.91} & 87.53 & 56.57 & 83.23 & 81.12 & 94.04 & 62.05 & 89.51 & 78.32
    \\
    \midrule
    \multirow{2}{*}{half-slice} 
    & baseline & 74.62 $\pm$ 0.78 & 31.62 $\pm$ 8.37 & 86.14 & 44.23 & 79.09 & 78.46 & 93.50 & 55.78 & 84.54 & 75.24
    \\
    & \ours & \best{77.37 $\pm$ 0.40} & \best{29.56 $\pm$ 1.09} & 86.89 & 55.96 & 82.15 & 78.63 & 94.34 & 57.36 & 86.60 & 77.05
    \\
    \midrule
    \multirow{2}{*}{half-vol} 
    & baseline & 71.08 $\pm$ 0.90 & 46.83 $\pm$ 2.91 & 84.38 & 46.71 & 78.19 & 74.55 & 92.02 & 48.03 & 76.28 & 68.47
    \\
    & \ours & \best{73.81 $\pm$ 0.94} & \best{35.33 $\pm$ 0.92} & 84.37 & 48.14 & 80.32 & 77.39 & 93.23 & 52.78 & 83.50 & 70.79
    \\
    \midrule
    \multirow{2}{*}{half-sparse} 
    & baseline & 31.74 $\pm$ 2.78 & 69.72 $\pm$ 1.37 & 65.71 & 8.33 & 59.46 & 51.59 & 51.18 & 10.72 & 6.92 & 0.00
    \\
    & \ours & \best{41.03 $\pm$ 2.12} & \best{59.04 $\pm$ 12.32} & 71.27 & 8.33 & 69.14 & 63.09 & 64.29 & 17.74 & 30.77 & 3.57
    \\
    \bottomrule
    \end{tabular}
    \end{adjustbox}
\end{table*}

\paragraph{Comparison with hard-thresholding algorithms.}
\Cref{tab:trim} illustrates the empirical advantage of \ours over the hard-thresholding algorithms, as suggested in \Cref{rmk:relation_hard_thresholding}. In particular, we consider the following hard-thresholding algorithms:
\begin{enumerate*}[label=(\roman*)]
    \item \b{trim-train} learns only from slices with at least one non-background pixel and trims the rest in each iteration on the fly;
    \item \b{trim-ratio} ranks the cross-entropy loss $\lossce\rbr{\theta_t;(\xb_i,\yb_i)}$ in each iteration (mini-batch) and trims samples with the lowest cross-entropy losses at a fixed ratio -- the ratio of all-background slices in the full training set ($1-\frac{1280}{2211} \approx 0.42$); 
    \item \b{ACR} further incorporates the data augmentation consistency regularization directly via the addition of $\regdac\rbr{\theta_t;\xb_i,A_{i,1},A_{i,2}}$ without reweighting;
    \item \b{pseudo-\ours} simulates the sample weights $\betab$ at the saddle point and learns via $\lossce\rbr{\theta_t;(\xb_i,\yb_i)}$ on slices with at least one non-background pixel while via $\regdac\rbr{\theta_t;\xb_i,A_{i,1},A_{i,2}}$ otherwise.
\end{enumerate*}
We see that naive incorporation of \b{ACR} brings less observable boosts to the hard-thresholding methods. Therefore, the deliberate combination via reweighting in \ours is essential for performance improvement.

\begin{table*}[!ht]
    \caption{\ours versus hard-thresholding algorithms with TransUNet on Synapse.}\vspace{-.5em}
    \label{tab:trim}
    \centering
    \begin{adjustbox}{width=\textwidth}
    \begin{tabular}{l|ccccccc}
    \toprule
    \multirow{2}{*}{Method} & \multirow{2}{*}{baseline} & \multicolumn{2}{c}{trim-train} & \multicolumn{2}{c}{trim-ratio} & \multirow{2}{*}{pseudo-\ours}  & \multirow{2}{*}{\ours} 
    \\ 
    \cline{3-6} 
    & & & +ACR &  & +ACR & \\
    \midrule
    DSC $\uparrow$ & 76.66 $\pm$ 0.88 & 76.80 $\pm$ 1.13 & 78.42 $\pm$ 0.17 & 76.49 $\pm$ 0.16 & 77.71 $\pm$ 0.56 & 77.72 $\pm$ 0.65 & \best{79.04 $\pm$ 0.21} \\
    HD95 $\downarrow$ & 29.23 $\pm$ 1.90 & 32.05 $\pm$ 2.34 & 27.84 $\pm$ 1.16 & 31.96 $\pm$ 2.60 & 28.51 $\pm$ 2.66 & 28.45 $\pm$ 1.18 & \best{27.39 $\pm$ 1.91} \\
    \bottomrule
    \end{tabular}
    \end{adjustbox}
\end{table*}

\paragraph{Segmentation on ACDC.}
Performance improvements granted by \ours are also observed on the ACDC dataset (\Cref{tab:acdc}). We defer detailed visualization of ACDC segmentation to \Cref{apx:additional_experiment}.
% \Cref{fig:acdc}. 
\begin{table*}[!ht]
    \caption{\ours with TransUNet trained on ACDC.}\vspace{-.5em}
    \label{tab:acdc}
    \centering
    \begin{adjustbox}{width=.6\textwidth} 
    \begin{tabular}{l|cc|ccc}
    \toprule
    Method & DSC $\uparrow$ & HD95 $\downarrow$  & RV & Myo & LV \\
    \midrule
    TransUNet &	89.40 $\pm$ 0.22 & 2.55 $\pm$ 0.37 & 89.17 & 83.24 & 95.78 \\
    \ours (ours) &	\best{90.67 $\pm$ 0.27} & \best{1.45 $\pm$ 0.55} & 90.00 & 85.94 & 96.06 \\
    \bottomrule
    \end{tabular}
    \end{adjustbox}
\end{table*}

\subsection{Ablation Study}
\paragraph{On the influence of consistency regularization.} 
To illustrate the role of consistency regularization in \ours, we consider the \b{reweight-only} scenario with $\lambdac=0$ such that $\regdac\rbr{\theta_t;\xb_i,A_{i,1},A_{i,2}} \equiv 0$ and therefore $\iloc{\bb}{2}$ (\Cref{alg:srw_dac} line 7) remains intact. With zero consistency regularization in \ours, reweighting alone brings little improvement (\Cref{tab:ablation}).

\paragraph{On the influence of sample reweighting.} 
We then investigate the effect of sample reweighting under different reweighting learning rates $\eta_{\betab}$ (recall \Cref{alg:srw_dac}):
\begin{enumerate*}[label=(\roman*)]
    \item \b{ACR-only} for $\eta_{\betab} = 0$ (equivalent to the naive addition of $\regdac\rbr{\theta_t;\xb_i,A_{i,1},A_{i,2}}$),
    \item \b{\ours-0.01} for $\eta_{\betab} = 0.01$, and 
    \item \b{\ours-1.0} for $\eta_{\betab} = 1.0$.
\end{enumerate*} 
As \Cref{tab:ablation} implies, when removing reweighting from \ours, augmentation consistency regularization alone improves DSC slightly from $76.28$ (baseline) to $77.89$ (ACR-only), whereas \ours boosts DSC to $79.12$ (\ours-1.0) with a proper choice of $\eta_{\betab}$.

\begin{table*}[!ht]
    \caption{Ablation study of \ours with TransUNet trained on Synapse.}\vspace{-.5em}
    \label{tab:ablation}
    \centering
    \begin{adjustbox}{width=1\textwidth}  
    \begin{tabular}{l|cc|cccccccc}
    \toprule
    Method &  DSC $\uparrow$ & HD95 $\downarrow$ & Aorta & Gallbladder & Kidney (L) & Kidney (R) & Liver & Pancreas & Spleen & Stomach \\
    \midrule
    baseline &	76.66 $\pm$ 0.88 & 29.23 $\pm$ 1.90 & 87.06 & 55.90 & 	81.95 & 75.58 & 94.29 & 56.30 & 86.05 & 76.17 \\
    reweight-only & 76.91 $\pm$ 0.88 & 30.92 $\pm$ 2.37 & 87.18 & 52.89	& 82.15 & 77.11	& 94.15 & 58.35	& 86.36 & 77.08 \\
    ACR-only & 78.01 $\pm$ 0.62 & 27.78 $\pm$ 2.80 & 87.51 & 58.79 & 83.39 & 79.26 & 94.70 & 58.99 & 86.02 & 75.43 \\
    \ours-0.01 & 77.75 $\pm$ 0.23 &	28.02 $\pm$ 3.50 & 87.33 & 56.68 & 83.35 & 78.53 & 94.45 & 57.02 & 87.72 & 76.94 \\
    \ours-1.0 & \best{79.04 $\pm$ 0.21} & \best{27.39 $\pm$ 1.91} & 87.53 & 56.57 & 83.23 & 81.12 & 94.04 & 62.05 & 89.51 & 78.32 \\
    \bottomrule
    \end{tabular}
    \end{adjustbox}
\end{table*}

\section{Discussion}

In this paper, we explore the information imbalance commonly observed in medical image segmentation and exploit the information in features of label-sparse samples via \ours, an adaptively weighted online optimization algorithm. 
\ours can be viewed as a careful combination of adaptive sample reweighting and data augmentation consistency regularization.
By casting the information imbalance among samples as a concept shift in the data distribution, we leverage the unsupervised data augmentation consistency regularization on the encoder layer outputs (of UNet-style architectures) as a natural reference for distinguishing the label-sparse and label-dense samples via the comparisons against the supervised average cross-entropy loss.
We formulate such comparisons as a weighted augmentation consistency (WAC) regularization problem and propose \ours for iterative and smooth separation of samples from different subpopulations with a convergence guarantee.
Our experiments on various medical image segmentation tasks with different UNet-style architectures empirically demonstrate the effectiveness of \ours not only in improving the segmentation performance and sample efficiency but also in enhancing the distributional robustness to concept shifts.

\paragraph{Limitations and future directions.}
From an algorithmic perspective, a limitation of this work is the utilization of the encoder layer outputs $\phi_\theta\rbr{\cdot}$ for data augmentation consistency regularization, which resulted in \ours being specifically tailored to UNet-style backbones. However, our method can be generalized to other architectures in principle by selecting a representation extractor in the network that
\begin{enumerate*}[label=(\roman*)]
    \item well characterizes the marginal distribution of features $P\rbr{\xb}$
    \item while being robust to the concept shift in $P\rsep{\yb}{\xb}$.
\end{enumerate*}
Further investigation into such generalizations is a promising avenue for future research.

Meanwhile, noticing the prevalence of concept shifts in natural data, especially for dense prediction tasks like segmentation and detection, we hope to extend the application/idea of \ours beyond medical image segmentation as a potential future direction.

\paragraph{Acknowledgement.}
R. Ward was partially supported by
AFOSR MURI FA9550-19-1-0005, 
NSF DMS 1952735%{\color{magenta}this is the grant with Gunnar Yijun was supported on}
, NSF HDR1934932, and 
NSF 2019844%{\color{magenta} this is the AI institute grant Gail was supported on}
. Y. Dong was supported by NSF DMS 1952735. Y. Xie was supported by NSF 2019844. The authors wish to thank Qi Lei and Xiaoxia Wu for valuable discussions and Jieneng Chen for generously providing preprocessed medical image segmentation datasets.

\bibliographystyle{abbrvnat}
\bibliography{ref}

\clearpage
\appendix
\section{Separation of Label-sparse and Label-dense Samples}\label{subapx:spontaneous_separation_sparse_dense}

\begin{proof}[Proof of \Cref{prop:spontaneous_separation_sparse_dense}]
    We first observe that, since 
    $\lossce\rbr{\theta;\rbr{\xb,\yb}}$ and $\regdac\rbr{\theta;\xb,A_{1},A_{2}}$ are convex and continuous in $\theta$ for all $(\xb,\yb) \in \Xcal \times \Ycal$ and $A_1,A_2 \sim \Acal^2$, for all $i \in [n]$, $\wh L^\wdac_i\rbr{\theta,\betab}$ is continuous, convex in $\theta$, and affine (thus concave) in $\betab$; and therefore so is $\wh L^\wdac \rbr{\theta,\betab}$. 
    Then with the compact and convex domains $\theta \in \Fnb{\gamma}$ and $\betab \in [0,1]^n$, Sion's minimax theorem~\citep{sion1958minimax} suggests the minimax equality,
    \begin{align}\label{eq:pf_minimax_optimal}
        \min_{\theta \in \Fnb{\gamma}}~ \max_{\betab \in [0,1]^n}~ \wh L^\wdac\rbr{\theta,\betab} = \max_{\betab \in [0,1]^n}~ \min_{\theta \in \Fnb{\gamma}}~ \wh L^\wdac\rbr{\theta,\betab},
    \end{align}
    where $\inf, \sup$ can be replaced by $\min, \max$ respectively due to compactness of the domains.

    Further, by the continuity and convexity-concavity of $\wh L^\wdac \rbr{\theta,\betab}$, the pointwise maximum $\max_{\betab \in [0,1]^n}\wh L^\wdac\rbr{\theta,\betab}$ is lower semi-continuous and convex in $\theta$ while the pointwise minimum $\min_{\theta \in \Fnb{\gamma}}\wh L^\wdac\rbr{\theta,\betab}$ is upper semi-continuous and concave in $\betab$. Then via Weierstrass' theorem (\cite{bertsekas2009convex}, Proposition 3.2.1), there exist $\wh\theta^\wdac \in \Fnb{\gamma}$ and $\wh\betab \in [0,1]^n$ that achieve the minimax optimal by minimizing $\max_{\betab \in [0,1]^n}\wh L^\wdac\rbr{\theta,\betab}$ and maximizing $\min_{\theta \in \Fnb{\gamma}}\wh L^\wdac\rbr{\theta,\betab}$. Along with \Cref{eq:pf_minimax_optimal}, such $\rbr{\wh\theta^\wdac, \wh\betab}$ provides a saddle point for \Cref{eq:objective_dac_encoder_weighted} (\cite{bertsekas2009convex}, Proposition 3.4.1).

    Next, we show via contradiction that there exists a saddle point with $\wh\betab$ attained on a vertex $\wh\betab \in \cbr{0,1}^n$. Suppose the opposite, then for any saddle point $\rbr{\wh\theta^\wdac, \wh\betab}$, there must be an $i \in [n]$ with $\iloc{\wh\betab}{i} \in (0,1)$, where we have the following contradictions:
    \begin{enumerate}[label=(\roman*)]
        \item If $\lossce\rbr{\wh\theta^\wdac;\rbr{\xb_i,\yb_i}} < \regdac\rbr{\wh\theta^\wdac;\xb_i,A_{i,1},A_{i,2}}$, decreasing $\iloc{\wh\betab}{i}>0$ to $\iloc{\wh\betab'}{i}=0$ leads to $\wh L^\wdac\rbr{\wh\theta^\wdac, \wh\betab'} > \wh L^\wdac\rbr{\wh\theta^\wdac, \wh\betab}$, contradicting \Cref{eq:saddle_point_def}.
        \item If $\lossce\rbr{\wh\theta^\wdac;\rbr{\xb_i,\yb_i}} > \regdac\rbr{\wh\theta^\wdac;\xb_i,A_{i,1},A_{i,2}}$, increasing $\iloc{\wh\betab}{i}<1$ to $\iloc{\wh\betab'}{i}=1$ again leads to $\wh L^\wdac\rbr{\wh\theta^\wdac, \wh\betab'} > \wh L^\wdac\rbr{\wh\theta^\wdac, \wh\betab}$, contradicting \Cref{eq:saddle_point_def}.
        \item If $\lossce\rbr{\wh\theta^\wdac;\rbr{\xb_i,\yb_i}} = \regdac\rbr{\wh\theta^\wdac;\xb_i,A_{i,1},A_{i,2}}$, $\iloc{\wh\betab}{i}$ can be replaced with any value in $[0,1]$, including $0,1$.
    \end{enumerate}
    Therefore, there must be a saddle point $\rbr{\wh\theta^\wdac, \wh\betab}$ with $\wh\betab \in \cbr{0,1}^n$ such that
    \begin{align*}
        \iloc{\betab}{i} = \iffun{\lossce\rbr{\wh\theta^\wdac;\rbr{\xb_i,\yb_i}} > \regdac\rbr{\wh\theta^\wdac;\xb_i,A_{i,1},A_{i,2}}}.
    \end{align*}

    Finally, it remains to show that \whp over $\cbr{\rbr{\xb_i, \yb_i}}_{i \in [n]}\sim P_\xi^n$ and $\cbr{\rbr{A_{i,1}, A_{i,2}}}_{i \in [n]} \sim \Acal^{2n}$,
    \begin{enumerate}[label=(\roman*)]
        \item $\lossce\rbr{\wh\theta^\wdac;\rbr{\xb_i,\yb_i}} \le \regdac\rbr{\wh\theta^\wdac;\xb_i,A_{i,1},A_{i,2}}$ for all $\rbr{\xb_i,\yb_i} \sim P_0$; and
        \item $\lossce\rbr{\wh\theta^\wdac;\rbr{\xb_i,\yb_i}} > \regdac\rbr{\wh\theta^\wdac;\xb_i,A_{i,1},A_{i,2}}$ for all $\rbr{\xb_i,\yb_i} \sim P_1$; 
    \end{enumerate}
    which leads to $\iloc{\betab}{i}=\iffun{\rbr{\xb_i,\yb_i} \sim P_1}$ \whp as desired.
    To illustrate this, we begin by observing that when $P_0$ and $P_1$ are $n$-separated (\Cref{ass:separability_sparse_dense}), since $\wh\theta^\wdac \in \Fnb{\gamma}$, there exists some $\omega>0$ such that for each $i \in [n]$,
    \begin{align*}
        \PP\ssepp{\lossce\rbr{\wh\theta^\wdac;\rbr{\xb_i,\yb_i}} < \regdac\rbr{\wh\theta^\wdac;\xb_i,A_{i,1},A_{i,2}}}{\rbr{\xb_i, \yb_i} \sim P_0} \ge 1-\frac{1}{\Omega\rbr{n^{1+\omega}}},
    \end{align*}
    and 
    \begin{align*}
        \PP\ssepp{\lossce\rbr{\wh\theta^\wdac;\rbr{\xb_i,\yb_i}} > \regdac\rbr{\wh\theta^\wdac;\xb_i,A_{i,1},A_{i,2}}}{\rbr{\xb_i, \yb_i} \sim P_1} \ge 1-\frac{1}{\Omega\rbr{n^{1+\omega}}}.
    \end{align*}
    Therefore by the union bound over the set of $n$ samples $\cbr{\rbr{\xb_i, \yb_i}}_{i \in [n]}\sim P_\xi^n$,
    \begin{align}\label{eq:pf_union_bound_P0}
        \PP\sbr{\lossce\rbr{\wh\theta^\wdac;\rbr{\xb_i,\yb_i}} < \regdac\rbr{\wh\theta^\wdac;\xb_i,A_{i,1},A_{i,2}} ~\forall~ \rbr{\xb_i, \yb_i} \sim P_0} \ge 1-\frac{1}{\Omega\rbr{n^\omega}},
    \end{align}
    and 
    \begin{align}\label{eq:pf_union_bound_P1}
        \PP\sbr{\lossce\rbr{\wh\theta^\wdac;\rbr{\xb_i,\yb_i}} > \regdac\rbr{\wh\theta^\wdac;\xb_i,A_{i,1},A_{i,2}} ~\forall~ \rbr{\xb_i, \yb_i} \sim P_1} \ge 1-\frac{1}{\Omega\rbr{n^\omega}}.
    \end{align}
    Applying the union bound again on \Cref{eq:pf_union_bound_P0} and \Cref{eq:pf_union_bound_P1}, we have the desired condition holds with probability $1-\Omega\rbr{n^\omega}^{-1}$, \ie, \whp.
\end{proof}

\section{Convergence of \ours}\label{subapx:convergence_wdac}

Recall the underlying function class $\Fcal \ni f_\theta$ parameterized by some $\theta \in \Fcal_\theta$ that we aim to learn for the pixel-wise classifier $h_\theta = \argmax_{k \in [K]} \iloc{f_\theta\rbr{\xb}}{j,:}$, $j \in [d]$:
\begin{align}
    \Fcal = \csepp{f_\theta = \phi_\theta \circ \psi_\theta}{\phi_\theta: \Xcal \to \Zcal,~ \psi_\theta:\Zcal \to [0,1]^{d \times K}},
\end{align}
where $\phi_\theta, \psi_\theta$ correspond to the encoder and decoder functions. Formally, we consider an inner product space of parameters $\rbr{\Fcal_\theta, \abr{\cdot,\cdot}_\Fcal}$ with the induced norm $\nbr{\cdot}_\Fcal$ and dual norm $\nbr{\cdot}_{\Fcal,*}$. 

For any $d \in \N$, let $\Delta_d^n \dfeq \csepp{\sbr{\betab_1;\dots;\betab_n} \in [0,1]^{n \times d}}{\nbr{\betab_i}_1 = 1 ~\forall~ i \in [n]}$. Then \Cref{eq:objective_dac_encoder_weighted} can be reformulated as:
\begin{align}\label{eq:objective_dac_encoder_weighted_l1_ball_formulation}
    &\wh\theta^\wdac, \wh\betab = 
    \argmin_{\theta \in \Fnb{\gamma}}~
    \argmax_{\Bb \in \Delta_2^n}~ \cbr{\wh L^\wdac\rbr{\theta,\Bb} \dfeq \frac{1}{n} \sum_{i=1}^n \wh L^\wdac_i\rbr{\theta,\Bb}},
    \\
    &\wh L^\wdac_i \rbr{\theta,\Bb} \dfeq 
    \iloc{\Bb}{i,1} \cdot \lossce\rbr{\theta;\rbr{\xb_i,\yb_i}} + 
    \iloc{\Bb}{i,2} \cdot \regdac\rbr{\theta;\xb_i,A_{i,1},A_{i,2}}. \nonumber
\end{align}

\begin{proposition}[Convergence (formal restatement of \Cref{prop:convergence_srw_dac})]\label{prop:convergence_srw_dac_formal}
    Assume that 
    $\lossce\rbr{\theta;\rbr{\xb,\yb}}$ and $\regdac\rbr{\theta;\xb,A_{1},A_{2}}$ are convex and continuous in $\theta$ for all $(\xb,\yb) \in \Xcal \times \Ycal$ and $A_1,A_2 \sim \Acal^2$, and that $\Fnb{\gamma} \subset \Fcal_\theta$ is convex and compact. 
    If there exist
    \begin{enumerate}[label=(\roman*)]
        \item $C_{\theta,*} > 0$ such that $\frac{1}{n} \sum_{i=1}^n \nbr{\nabla_\theta \wh L_i^\wdac\rbr{\theta,\Bb}}_{\Fcal,*}^2 \le C_{\theta,*}^2$ for all $\theta \in \Fnb{\gamma}$, $\Bb \in \Delta_2^n$ and 
        \item $C_{\Bb,*} > 0$ such that $\frac{1}{n} \sum_{i=1}^n \max\cbr{\lossce\rbr{\theta;\rbr{\xb_i,\yb_i}}, \regdac\rbr{\theta;\xb_i,A_{i,1},A_{i,2}} }^2 \le C_{\Bb,*}^2$ for all $\theta \in \Fnb{\gamma}$,
    \end{enumerate} 
    then with $\eta_\theta = \eta_{\Bb} = 2 \Big/ \sqrt{5T \rbr{\gamma^2 C_{\theta,*}^2 + 2 n C_{\Bb,*}^2}}$, \Cref{alg:srw_dac} provides the convergence guarantee for the duality gap $\Ecal\rbr{\overline{\theta}_T, \overline{\Bb}_T} \dfeq \max_{\Bb \in \Delta_2^n} \wh L^\wdac\rbr{\overline{\theta}_T, \Bb} - \min_{\theta \in \Fnb{\gamma}} \wh L^\wdac\rbr{\theta, \overline{\Bb}_T}$:
    \begin{align*}
        \E\sbr{\Ecal\rbr{\overline{\theta}_T, \overline{\Bb}_T} }
        \le 2 \sqrt{\frac{5 \rbr{\gamma^2 C_{\theta,*}^2 + 2 n C_{\Bb,*}^2}}{T}},
    \end{align*}
    where $\overline{\theta}_T = \frac{1}{T} \sum_{t=1}^T \theta_t$ and $\overline{\Bb}_T = \frac{1}{T} \sum_{t=1}^T \Bb_t$.
\end{proposition} 

\begin{proof}[Proof of \Cref{prop:convergence_srw_dac_formal}]
The proof is an application of the standard convergence guarantee for the online mirror descent on saddle point problems, as recapitulated in \Cref{lemma:mirror_descent_convergence_saddle_point}.  

Specifically, for $\Bb \in \Delta_2^n$, we use the norm $\nbr{\Bb}_{1,2} \dfeq \sqrt{\sum_{i=1}^n \rbr{\sum_{j=1}^2 \abbr{\iloc{\Bb}{i,j}}}^2}$ with its dual norm $\nbr{\Bb}_{1,2,*} \dfeq \sqrt{\sum_{i=1}^n \rbr{\max_{j \in [2]} \abbr{\iloc{\Bb}{i,j}}}^2}$. 
We consider a mirror map $\varphi_{\Bb}: [0,1]^{n \times 2} \to \R$ such that $\varphi_{\Bb}\rbr{\Bb} = \sum_{i=1}^n \sum_{j=1}^2 \iloc{\Bb}{i,j} \log \iloc{\Bb}{i,j}$. We observe that, since $\iloc{\Bb}{i,:}, \iloc{\Bb'}{i,:} \in \Delta_2$ for all $i \in [n]$,
\begin{align*}
    D_{\varphi_{\Bb}}\rbr{\Bb, \Bb'} 
    = \sum_{i=1}^n \sum_{j=1}^2 \iloc{\Bb}{i,j} \log \frac{\iloc{\Bb}{i,j}}{\iloc{\Bb'}{i,j}} 
    \ge \frac{1}{2} \sum_{i=1}^n \rbr{\sum_{j=1}^2 \abbr{\iloc{\Bb}{i,j} - \iloc{\Bb'}{i,j}}}^2
    = \frac{1}{2} \nbr{\Bb - \Bb'}_{1,2}^2,
\end{align*}
and therefore $\varphi_{\Bb}$ is $1$-strongly convex with respect to $\nbr{\cdot}_{1,2}$. With such $\varphi_\Bb$, we have the associated Fenchel dual $\varphi^*_\Bb\rbr{\Gb} = \sum_{i=1}^n \log \rbr{\sum_{j=1}^2 \exp\rbr{\iloc{\Gb}{i,j}}}$, along with the gradients
\begin{align*}
    \iloc{\nabla \varphi_\Bb\rbr{\Bb}}{i,j} = 1 + \log \iloc{\Bb}{i,j},
    \quad
    \iloc{\nabla \varphi^*_\Bb\rbr{\Gb}}{i,j} = \frac{\exp\rbr{\iloc{\Gb}{i,j}}}{\sum_{j=1}^2 \exp\rbr{\iloc{\Gb}{i,j}}},
\end{align*}
such that the mirror descent update on $\Bb$ is given by
\begin{align*}
    \iloc{\rbr{\Bb_{t+1}}}{i,j} 
    = &\nabla\varphi_\Bb^*\rbr{\nabla\varphi_\Bb\rbr{\Bb_t} - \eta_\Bb \cdot \nabla_\Bb\wh L^\wdac_{i_t}\rbr{\theta_t, \Bb_{t}}}
    \\
    = &\frac{\iloc{\rbr{\Bb_{t}}}{i,j} \exp\rbr{\eta_\Bb \cdot \iloc{\rbr{\nabla_\Bb \wh L^\wdac_{i_t}\rbr{\theta_t, \Bb_t}}}{i,j}}}{\sum_{j=1}^2 \iloc{\rbr{\Bb_{t}}}{i,j} \exp\rbr{\eta_\Bb \cdot \iloc{\rbr{\nabla_\Bb \wh L^\wdac_{i_t}\rbr{\theta_t, \Bb_t}}}{i,j}}}.
\end{align*}    
For $i_t \sim [n]$ uniformly, the stochastic gradient with respect to $\Bb$ satisfies
\begin{align*}
    &\E_{i_t \sim [n]}\sbr{\nbr{\nabla_\Bb \wh L^\wdac_{i_t}\rbr{\theta_t, \Bb_t}}_{1,2,*}^2}
    \\
    = &\frac{1}{n} \sum_{i_t=1}^n \max\cbr{ \lossce\rbr{\theta_t;\rbr{\xb_{i_t},\yb_{i_t}}}, \regdac\rbr{\theta_t;\xb_{i_t},A_{i_t,1},A_{i_t,2}} }^2
    \le C_{\Bb,*}^2.
\end{align*}
Further, in the distance induced by $\varphi_\Bb$, we have
\begin{align*}
    R_{\Delta_2^n}^2 
    \dfeq \max_{\Bb \in \Delta_2^n} \varphi_{\Bb}\rbr{\Bb} - \min_{\Bb \in \Delta_2^n} \varphi_{\Bb}\rbr{\Bb}
    = 0 - \sum_{i=1}^n \sum_{j=1}^2 \frac{1}{2} \log \frac{1}{2}
    = n.
\end{align*}

Meanwhile, for $\theta \in \Fnb{\gamma}$, we consider the norm $\nbr{\theta}_\Fcal \dfeq \sqrt{\abr{\theta,\theta}_\Fcal}$ induced by the inner product that characterizes $\Fcal_\theta$, with the associated dual norm $\nbr{\cdot}_{\Fcal,*}$. We use a mirror map $\varphi_\theta:\Fcal_\theta \to \R$ such that $\varphi_\theta\rbr{\theta} = \frac{1}{2}\nbr{\theta-\theta^*}^2_{\Fcal}$. By observing that
\begin{align*}
    D_{\varphi_\theta}\rbr{\theta,\theta'} = \frac{1}{2} \nbr{\theta-\theta'}_{\Fcal}^2 \quad \forall~ \theta,\theta' \in \Fcal.
\end{align*}
we have $\varphi_\theta$ being $1$-strongly convex with respect to $\nbr{\cdot}_{\Fcal}$.
With the gradient of $\varphi_\theta$, $\nabla\varphi_\theta(\theta) = \theta-\theta^*$, and that of its Fenchel dual $\nabla\varphi_\theta^*(g)=g+\theta^*$, at the $(t+1)$-th iteration, we have
\begin{align*}
    \theta_{t+1} 
    = \nabla\varphi_\theta^*\rbr{\nabla\varphi_\theta\rbr{\theta_t} - \eta_\theta \cdot \nabla_\theta\wh L^\wdac_{i_t}\rbr{\theta_t, \Bb_{t+1}}}
    = \theta_t - \eta_\theta \cdot \nabla_\theta\wh L^\wdac_{i_t}\rbr{\theta_t, \Bb_{t+1}}.
\end{align*}
For $i_t \sim [n]$ uniformly, the stochastic gradient with respect to $f$ satisfies that
\begin{align*}
    \E_{i_t \sim [n]}\sbr{\nbr{\nabla_\theta \wh L^\wdac_{i_t}\rbr{\theta_t, \Bb_{t+1}}}_{\Fcal,*}^2}
    = \frac{1}{n} \sum_{i_t=1}^n \nbr{\nabla_\theta \wh L_{i_t}^\wdac\rbr{\theta_t,\Bb_{t+1}}}_{\Fcal,*}^2 
    \le C_{\theta,*}^2.
\end{align*}
Further, in light of the definition of $\Fnb{\gamma}$, since $\theta^* \in \Fnb{\gamma}$, with $\theta^* = \argmin_{\theta \in \Fnb{\gamma}} \varphi_\theta(\theta)$ and $\theta' = \argmax_{\theta \in \Fnb{\gamma}} \varphi_\theta(\theta)$, we have 
\begin{align*}
    R_{\Fnb{\gamma}}^2 
    \dfeq \max_{\theta \in \Fnb{\gamma}} \varphi_\theta\rbr{\theta} - \min_{\theta \in \Fnb{\gamma}} \varphi_\theta\rbr{\theta}
    = \frac{1}{2}\nbr{\theta' - \theta^*}_{\Fcal}^2 
    \le \frac{\gamma^2}{2}.
\end{align*}

Finally, leveraging \Cref{lemma:mirror_descent_convergence_saddle_point} completes the proof.
\end{proof}

We recall the standard convergence guarantee for online mirror descent on saddle point problems. In general, we consider a stochastic function $F:\Ucal \times \Vcal \times \Ical \to \R$ with the randomness of $F\rbr{u,v;i}$ on $i \in \Ical$. 
Overloading notation $\Ical$ both as the distribution of $i$ and as the support, we are interested in solving the saddle point problem on the expectation function
\begin{align}\label{eq:apx_omd_saddle_point_problem}
    \min_{u \in \Ucal} \max_{v \in \Vcal} f\rbr{u,v}
    \quad \t{where} \quad
    f\rbr{u,v} \dfeq \E_{i \sim \Ical}\sbr{F\rbr{u,v;i}}.
\end{align}

\begin{assumption}\label{ass:apx_omd_objective}
Assume that the stochastic objective satisfies the following:
\begin{enumerate}[label=(\roman*)]
    \item For every $i \in \Ical$, $F\rbr{\cdot,v,i}$ is convex for all $v \in \Vcal$ and $F\rbr{u,\cdot,i}$ is concave for all $u \in \Ucal$.
    \item The stochastic subgradients $G_u\rbr{u,v;i} \in \partial_u F\rbr{u,v;i}$ and $G_v\rbr{u,v;i} \in \partial_v F\rbr{u,v;i}$ with respect to $u$ and $v$ evaluated at any $\rbr{u,v} \in \Ucal \times \Vcal$ provide unbiased estimators for some respective subgradients of the expectation function: for any $\rbr{u,v} \in \Ucal \times \Vcal$, there exist some $g_u\rbr{u,v} \dfeq \E_{i \sim \Ical}\sbr{G_u\rbr{u,v;i}} \in \partial_u f\rbr{u,v}$ and $g_v\rbr{u,v} \dfeq \E_{i \sim \Ical}\sbr{G_v\rbr{u,v;i}} \in \partial_v f\rbr{u,v}$.
    \item Let $\nbr{\cdot}_\Ucal$ and $\nbr{\cdot}_\Vcal$ be arbitrary norms that are well-defined on $\Ucal$ and $\Vcal$, while $\nbr{\cdot}_{\Ucal,*}$ and $\nbr{\cdot}_{\Vcal,*}$ be their respective dual norms. There exist constants $C_{u,*}, C_{v,*} > 0$ such that
    \begin{align*}
        \E_{i \sim \Ical}\sbr{\nbr{G_u\rbr{u,v;i}}_{\Ucal,*}^2} \le C_{u,*}^2 ~\land~
        \E_{i \sim \Ical}\sbr{\nbr{G_v\rbr{u,v;i}}_{\Vcal,*}^2} \le C_{v,*}^2
        \quad\forall~ \rbr{u,v} \in \Ucal \times \Vcal.
    \end{align*}    
\end{enumerate} 
\end{assumption}   

For online mirror descent, we further introduce two mirror maps that induce distances on $\Ucal$ and $\Vcal$, respectively.
\begin{assumption}\label{ass:apx_omd_mirror_map}
Let $\varphi_u: \Dcal_u \to \R$ and $\varphi_v: \Dcal_v \to \R$ satisfy the following:
\begin{enumerate}[label=(\roman*)]
    \item $\Ucal \subseteq \Dcal_u \cup \partial \Dcal_u$, $\Ucal \cap \Dcal_u \neq \emptyset$ and $\Vcal \subseteq \Dcal_v \cup \partial \Dcal_v$, $\Vcal \cap \Dcal_v \neq \emptyset$.
    \item $\varphi_u$ is $\rho_u$-strongly convex with respect to $\nbr{\cdot}_\Ucal$; $\varphi_v$ is $\rho_v$-strongly convex with respect to $\nbr{\cdot}_\Vcal$.
    \item $\lim_{u \to \partial \Dcal_u} \nbr{\nabla \varphi_u(u)}_{\Ucal,*} = \lim_{v \to \partial \Dcal_v} \nbr{\nabla \varphi_v(v)}_{\Vcal,*} = +\infty$.
\end{enumerate}
\end{assumption}
Given the learning rates $\eta_u, \eta_v$, in each iteration $t=1,\dots,T$, the online mirror descent samples $i_t \sim \Ical$ and updates 
\begin{align}\label{eq:apx_omd_update}
    &v_{t+1} = \argmin_{v \in \Vcal}\ -\eta_v \cdot G_v\rbr{u_t,v_t;i_t}^\top v + D_{\varphi_v}\rbr{v, v_t},
    \nonumber \\
    &u_{t+1} = \argmin_{u \in \Ucal}\ \eta_u \cdot G_u\rbr{u_t,v_{t+1};i_t}^\top u + D_{\varphi_u}\rbr{u, u_t},
\end{align}
where $D_{\varphi}\rbr{w,w'} = \varphi(w)-\varphi(w')-\nabla\varphi(w')^\top(w-w')$ denotes the Bregman divergence.

We measure the convergence of the saddle point problem in the duality gap:
\begin{align*}
    \Ecal\rbr{\overline{u}_T, \overline{v}_T} \dfeq \max_{v \in \Vcal} f\rbr{\overline{u}_T, v} - \min_{u \in \Ucal} f\rbr{u, \overline{v}_T}
\end{align*}
such that, with
\begin{align*}
    R_\Ucal \dfeq \sqrt{\max_{u \in \Ucal \cap \Dcal_u} \varphi_u(u) - \min_{u \in \Ucal \cap \Dcal_u} \varphi_u(u)}
    \quad\t{and}\quad
    R_\Vcal \dfeq \sqrt{\max_{v \in \Vcal \cap \Dcal_v} \varphi_v(v) - \min_{v \in \Vcal \cap \Dcal_v} \varphi_v(v)},
\end{align*}
the online mirror descent converges as follows.
\begin{lemma}[\citep{nemirovski2009omd} (3.11)]\label{lemma:mirror_descent_convergence_saddle_point}
Under \Cref{ass:apx_omd_objective} and \Cref{ass:apx_omd_mirror_map}, when taking constant learning rates $\eta_u = \eta_v = 2 \Big/ \sqrt{5T \rbr{\frac{2 R_\Ucal^2}{\rho_u} C_{u,*}^2 + \frac{2 R_\Vcal^2}{\rho_v} C_{v,*}^2}}$, with $\overline{u}_T = \frac{1}{T} \sum_{t=1}^T u_t$ and $\overline{v}_T = \frac{1}{T} \sum_{t=1}^T v_t$,
\begin{align*}
    \E\sbr{\Ecal\rbr{\overline{u}_T, \overline{v}_T} }
    \le 2 \sqrt{\frac{10 \rbr{\rho_v R_\Ucal^2 C_{u,*}^2 + \rho_u R_\Vcal^2 C_{v,*}^2}}{\rho_u \rho_v \cdot T}}.
\end{align*}
\end{lemma}

\begin{example}[Binary linear pixel-wise classifiers with convex and continuous objectives]\label{ex:convex_continuous}
    We consider a pixel-wise binary classification problem with $\Xcal=[0,1]^d$, augmentations $A:\Xcal \to \Xcal$ for all $A \sim \Acal$, and a class of linear ``UNets'',
    \begin{align*}
        \Fcal = \csepp{f_\theta:\Xcal \to [0,1]^d}{f_\theta\rbr{\xb} = \sigma\rbr{\thetab_d \thetab_e^\top \xb} = \psi_\theta\rbr{\phi_\theta\rbr{\xb}},~ \phi_\theta\rbr{\xb} = \frac{1}{\sqrt{d}}\thetab_e^\top \xb},
    \end{align*}
    where the parameter space $\theta = \rbr{\thetab_e, \thetab_d} \in \Fcal_\theta = \SSS^{d-1} \times \SSS^{d-1}$ is equipped with the $\ell_2$ norm $\nbr{\theta}_{\Fcal} = \rbr{\nbr{\thetab_e}_2^2 + \nbr{\thetab_d}_2^2}^{1/2}$; 
    $\sigma:\R^d \to [0,1]^d$ denotes entry-wise application of the sigmoid function $\sigma(z) = (1+e^{-z})^{-1}$; and 
    the latent space of encoder outputs $\rbr{\Zcal,\varrho}$ is simply the real line. 
    Given the data distribution $P_\xi$, we recall that $\theta^* = \argmin_{\theta \in \Fcal_\theta} \E_{\rbr{\xb,\yb}\sim P_\xi}\sbr{\lossce\rbr{\theta;(\xb,\yb)} }$ for all $\xi \in [0,1]$ and let $\Fnb{\gamma} = \csepp{\theta \in \Fcal_\theta}{\nbr{\theta - \theta^*}_{\Fcal} \le \gamma}$ for some $\gamma = O\rbr{1/\sqrt{d}}$.
    We assume that $\abbr{\xb^\top \thetab_e^*} = O(1)$ for all $\xb \in \Xcal$. 
    Then, $\lossce\rbr{\theta;\rbr{\xb,\yb}}$ and $\regdac\rbr{\theta;\xb,A_1,A_2}$ are convex and continuous in $\theta$ for all $(\xb,\yb) \in \Xcal \times [K]^d$, $A_1,A_2 \sim \Acal^2$; while $C_{\theta,*} \le \max\rbr{2\sqrt{2}, 2\lambdac}$ and $C_{\betab,*} \le \max\rbr{O(1), 2\lambdac}$.
\end{example}

\begin{proof}[Rationale for \Cref{ex:convex_continuous}]
    Let $\yb_k = \iffun{\yb=k}$ entry-wise for $k=0,1$. We would like to show that, for any given $(\xb,\yb) \in \Xcal \times [K]^d$, $A_1,A_2 \sim \Acal^2$,
    \begin{align*}
        &\lossce\rbr{\theta} 
        = -\frac{1}{d} \rbr{\yb_1^\top \log \sigma\rbr{\thetab_d \thetab_e^\top \xb} + \yb_0^\top \log \sigma\rbr{-\thetab_d \thetab_e^\top \xb}},
        \\
        &\regdac\rbr{\theta} = \frac{\lambdac}{\sqrt{d}} \cdot \rbr{A_1(\xb)-A_2(\xb)}^\top \thetab_e
    \end{align*}
    are convex and continuous in $\theta = \rbr{\thetab_e,\thetab_d}$.

    First, we observe that $\regdac\rbr{\theta}$ is linear (and therefore convex and continuous) in $\theta$ for all $\xb \in \Xcal$, $A_1,A_2 \sim \Acal^2$, with
    \begin{align*}
        \nabla_{\thetab_e} \regdac\rbr{\theta} = \frac{\lambdac}{\sqrt{d}} \cdot \rbr{A_1(\xb)-A_2(\xb)},
        \quad
        \nabla_{\thetab_d} \regdac\rbr{\theta} = \b{0}
    \end{align*}
    such that $\nbr{\nabla_{\theta} \regdac\rbr{\theta}}_{\Fcal,*} \le 2 \lambdac$.
    
    Meanwhile, with $\zb\rbr{\theta} = \thetab_d \thetab_e^\top \xb$, we have $\lossce\rbr{\theta} = -\frac{1}{d} \rbr{\yb_1^\top \log \sigma\rbr{\zb\rbr{\theta}} + \yb_0^\top \log \sigma\rbr{-\zb\rbr{\theta}}}$ being convex and continuous in $\zb\rbr{\theta}$:
    \begin{align*}
        \nabla_{\zb}^2 \lossce\rbr{\theta} = \frac{1}{d} \diag\rbr{\sigma\rbr
        {\zb\rbr{\theta}}} \diag\rbr{1-\sigma\rbr
        {\zb\rbr{\theta}}} \ageq 0.
    \end{align*}
    Therefore, $\lossce\rbr{\theta}$ is convex and continuous in $\theta$ for all $\rbr{\xb,\yb} \in \Xcal \times [K]^d$:
    \begin{align*}
        \underbrace{\nabla_{\theta}^2 \lossce\rbr{\theta}}_{2d \times 2d} = 
        \bmat{\xb \thetab_d^\top \\ \rbr{\thetab_e^\top \xb} \Ib_{d}}
        \rbr{\frac{1}{d} \diag\rbr{\sigma\rbr
        {\zb\rbr{\theta}}} \diag\rbr{1-\sigma\rbr
        {\zb\rbr{\theta}}}}
        \bmat{\xb \thetab_d^\top & \rbr{\thetab_e^\top \xb} \Ib_{d}}
        \ageq 0,
    \end{align*}
    where $\Ib_d$ denotes the $d \times d$ identity matrix.
    Further, from the derivation, we have
    \begin{align*}
        &\nabla_{\thetab_e} \lossce\rbr{\theta} = \frac{1}{d} \thetab_d^\top \rbr{\sigma\rbr{\thetab_d \thetab_e^\top \xb}-\yb} \xb,
        \quad
        \nabla_{\thetab_d} \lossce\rbr{\theta} = \frac{\thetab_e^\top \xb}{d} \rbr{\sigma\rbr{\thetab_d \thetab_e^\top \xb}-\yb}
    \end{align*}
    such that $\nbr{\nabla_{\theta} \lossce\rbr{\theta}}_{\Fcal,*} = \sqrt{\nbr{\nabla_{\thetab_e} \lossce\rbr{\theta}}_2^2 + \nbr{\nabla_{\thetab_d} \lossce\rbr{\theta}}_2^2} \le 2 \sqrt{2}$.

    Finally, knowing $\nbr{\nabla_{\theta} \lossce\rbr{\theta}}_{\Fcal,*}  \le 2 \sqrt{2}$ and $\nbr{\nabla_{\theta} \regdac\rbr{\theta}}_{\Fcal,*} \le 2 \lambdac$, we have
    \begin{align*}
        \nbr{\nabla_\theta \wh L_i^\wdac\rbr{\theta,\betab}}_{\Fcal,*} 
        \le \iloc{\betab}{i} \nbr{\nabla_{\theta} \lossce\rbr{\theta}}_{\Fcal,*} + (1-\iloc{\betab}{i}) \nbr{\nabla_{\theta} \regdac\rbr{\theta}}_{\Fcal,*} \le \max\rbr{2\sqrt{2}, 2\lambdac}
    \end{align*}
    for all $i \in [n]$, and therefore,
    \begin{align*}
        C_{\theta,*} \le \max\rbr{2\sqrt{2}, 2\lambdac}.
    \end{align*}
    Besides, with
    \begin{align*}
        \regdac\rbr{\theta} \le \frac{\lambdac}{\sqrt{d}} \nbr{A_1(\xb)-A_2(\xb)}_2 \nbr{\thetab_e}_2 \le 2 \lambdac,
    \end{align*} 
    and since 
    \begin{align*}
        \iloc{\rbr{\thetab_d \thetab_e^\top \xb}}{j} 
        \le &\abbr{\xb^\top \thetab_e} \le \abbr{\xb^\top \rbr{\thetab_e - \thetab_e^*}} + \abbr{\xb^\top \thetab_e^*} \le \nbr{\xb}_2 \nbr{\thetab_e - \thetab_e^*}_2 + O(1)
        \\
        \le & \gamma \sqrt{d} + O(1) = O(1)
    \end{align*}
    for all $j \in [d]$, $\lossce\rbr{\theta} \le \log\rbr{1+e^{O(1)}} = O(1)$, we have 
    \begin{align*}
        C_{\betab,*} \le \max\rbr{O(1), 2\lambdac}.
    \end{align*}
\end{proof}

\section{Dice Loss for Pixel-wise Class Imbalance}\label{apx:dice}

With finite samples in practice, since the averaged cross-entropy loss (\Cref{eq:cross_entropy_loss}) weights each pixel in the image label equally, the pixel-wise class imbalance can become a problem. For example, the background pixels can be dominant in most of the segmentation labels, making the classifier prone to predict pixels as background. 

To cope with such vulnerability, \cite{chen2021transunet,cao2021swin,wong2018segmentation,taghanaki2019combo,yeung2022unified} propose to combine the cross-entropy loss with the \emph{dice loss}---a popular segmentation loss based on the overlap between true labels and their corresponding predictions in each class:
\begin{align}\label{eq:dice_loss}
    \lossdice\rbr{\theta;\rbr{\xb,\yb}} = 1 - \frac{1}{K} \sum_{k=1}^K \dsc\rbr{\iloc{f_\theta\rbr{\xb}}{:,k}, \iffun{\yb=k}},
\end{align}
where for any $\pb \in [0,1]^d$, $\qb \in \cbr{0,1}^d$, $\dsc\rbr{\pb,\qb} = \frac{2 \pb^\top \qb}{\nbr{\pb}_1 + \nbr{\qb}_1} \in [0,1]$ denotes the dice coefficient~\citep{milletari2016vnet,taghanaki2019deep}. Notice that by measuring the bounded dice coefficient for each of the $K$ classes individually, the dice loss tends to be robust to class imbalance. 

\cite{taghanaki2019combo} merges both dice and averaged cross-entropy losses via a convex combination. It is also a common practice to add a smoothing term in both the nominator and denominator of the DSC~\citep{russell2016artificial}.

Combining the dice loss (\Cref{eq:dice_loss}) with the weighted augmentation consistency regularization formulation (\Cref{eq:objective_dac_encoder_weighted}), in practice, we solve
\begin{align}\label{eq:objective_wac_dice}
    &\wh\theta^\wdac, \wh\betab \in 
    \argmin_{\theta \in \Fnb{\gamma}}~
    \argmax_{\betab \in [0,1]^n}~ 
    \cbr{\wh L^\wdac\rbr{\theta,\betab} \dfeq \frac{1}{n} \sum_{i=1}^n \wh L^\wdac_i\rbr{\theta,\betab}}
    \\
    &\wh L^\wdac_i \rbr{\theta,\betab} \dfeq 
    \lossdice\rbr{\theta;\rbr{\xb_i,\yb_i}} + 
    \iloc{\betab}{i} \cdot \lossce\rbr{\theta;\rbr{\xb_i,\yb_i}} + 
    (1-\iloc{\betab}{i}) \cdot \regdac\rbr{\theta;\xb_i,A_{i,1},A_{i,2}} \nonumber
\end{align}
with a slight modification in \Cref{alg:srw_dac} line 9:
\begin{align*}
    \theta_t \gets \theta_{t-1} - \eta_\theta \cdot \Big(
    \nabla_\theta\lossdice\rbr{\theta_{t-1};\rbr{\xb_{i_t},\yb_{i_t}}} + 
    \iloc{\rbr{\betab_t}}{i_t} \cdot \nabla_\theta\lossce\rbr{\theta_{t-1};\rbr{\xb_{i_t},\yb_{i_t}}}&
    \\
    + \rbr{1-\iloc{\rbr{\betab_t}}{i_t}} \cdot \nabla_\theta\regdac\rbr{\theta_{t-1};\xb_{i_t},A_{i_t,1},A_{i_t,2}}& \Big).
\end{align*}

\paragraph{On the influence of incorporating dice loss in experiments.} 
We note that, in the experiments, the dice loss $\lossdice$ is treated independently of \ours in \Cref{alg:srw_dac} via standard stochastic gradient descent. In particular for the comparison with hard-thresholding algorithms in \Cref{tab:trim}, we keep the updating on $\lossdice$ of the original untrimmed batch intact for both \b{trim-train} and \b{trim-ratio} to exclude the potential effect of $\lossdice$ that is not involved in reweighting.

\section{Implementation Details and Datasets}\label{app:imp}
We follow the official implementation of TransUNet\footnote{\href{https://github.com/Beckschen/TransUNet}{https://github.com/Beckschen/TransUNet}} for model training. We use the same optimizer (SGD with learning rate $0.01$, momentum $0.9$, and weight decay 1e-4). 
For the Synapse dataset, we train TransUNet for 150 epochs on the training dataset and evaluate the last-iteration model on the test dataset. For the ACDC dataset, we train TransUNet for 360 epochs in total, while validating models on the ACDC validation dataset for every 10 epochs and testing on the best model selected by the validation. 
The total number of training iterations (\ie, total number of batches) is set to be the same as that in the vanilla TransUNet~\citep{chen2021transunet} experiments. In particular, the results in \Cref{tab:synapse_sample_eff} are averages (and standard deviations) over $3$ arbitrary random seeds. The results in \Cref{tab:trim}, \Cref{tab:acdc}, and \Cref{tab:ablation} are given by the original random seed used in the TransUNet experiments.

\paragraph{Synapse multi-organ segmentation dataset (Synapse).}The Synapse dataset\footnote{See detailed description at \href{https://www.synapse.org/\#!Synapse:syn3193805/wiki/217789}{{https://www.synapse.org/\#!Synapse:syn3193805/wiki/217789}}} is multi-organ abdominal CT scans for medical image segmentation in the MICCAI 2015 Multi-Atlas Abdomen Labelling Challenge~\citep{chen2021transunet}. There are 30 cases of CT scans with variable sizes $(512 \times 512 \times 85 - 512 \times 512 \times 198)$, and slice thickness ranges from $2.5$mm to $5.0$mm. We use the pre-processed data provided by \cite{chen2021transunet} and follow their train/test split to use 18 cases for training and 12 cases for testing on 8 abdominal organs---aorta, gallbladder, left kidney (L), right kidney (R), liver, pancreas, spleen, and stomach. The abdominal organs were labeled by experience undergraduates and verified by a radiologist using MIPAV software according to the information from Synapse wiki page.

\paragraph{Automated cardiac diagnosis challenge dataset (ACDC).}The ACDC dataset\footnote{See detailed description at \href{https://www.creatis.insa-lyon.fr/Challenge/acdc/}{https://www.creatis.insa-lyon.fr/Challenge/acdc/}} is cine-MRI scans in the MICCAI 2017 Automated Cardiac Diagnosis Challenge. There are $200$ scans from $100$ patients, and each patient has two frames with slice thickness from $5$mm to $8$mm. We use the pre-processed data also provided by \cite{chen2021transunet} and follow their train/validate/test split to use $70$ patients' scans for training, $10$ patients' scans for validation, and 20 patients' scans for testing on three cardiac structures---left ventricle (LV), myocardium (MYO), and right ventricle (RV). The data were labeled by one clinical expert according to the description on ACDC dataset website.

\section{Additional Experimental Results}\label{apx:additional_experiment}

\subsection{Sample Efficiency and Robustness of \ours with UNet}\label{subsec:exp_unet}

In addition to the empirical evidence on TransUNet presented in \Cref{tab:synapse_sample_eff}, here, we demonstrate that the sample efficiency and distributional robustness of \ours extend to the more widely used UNet architecture. 
In \Cref{tab:synapse_sample_eff_unet}, analogous to \Cref{tab:synapse_sample_eff}, the experiments on the \b{full} and \b{half-slice} datasets provide evidence for the \emph{sample efficiency} of \ours compared to the baseline (ERM+SGD) on UNet. Meanwhile, the \emph{distributional robustness} of \ours with UNet is well illustrated by the \b{half-vol} and \b{half-sparse} experiments.

\begin{table}[ht]
    \caption{\ours with UNet trained on the full Synapse and its subsets
    }
    % \vspace{-.5em}
    \label{tab:synapse_sample_eff_unet}
    \centering
    \begin{adjustbox}{width=1\textwidth}  
    \begin{tabular}{ll|cc|cccccccc}
    \toprule
    Training & Method &  DSC $\uparrow$ & HD95 $\downarrow$ & Aorta & Gallbladder & Kidney (L) & Kidney (R) & Liver & Pancreas & Spleen & Stomach 
    \\
    \midrule
    \multirow{2}{*}{full} 
    & baseline & 74.04 $\pm$ 1.52 & 36.65 $\pm$ 0.33 & 84.93 & 55.59 & 77.59 & 70.92 & 92.21 & 55.01 & 82.87 & 73.21
    \\
    & \ours & \best{76.71 $\pm$ 0.62} & \best{30.67 $\pm$ 2.85}
    & 85.68 & 55.19 & 80.15 & 75.45 & 94.11 & 56.19 & 87.54 & 81.39
    \\
    \midrule
    \multirow{2}{*}{half-slice} 
    & baseline & 73.09 $\pm$ 0.10 & 40.05 $\pm$ 4.99 & 83.23 & 53.18 & 74.69 & 71.51 & 92.74 & 52.81 & 83.85 & 72.71 
    \\
    & \ours & \best{75.12 $\pm$ 0.78} & \best{29.26 $\pm$ 2.16} & 85.15 & 55.77 & 79.29 & 72.47 & 93.71 & 54.93 & 86.09 & 73.53
    \\
    \midrule
    \multirow{2}{*}{half-vol} 
    & baseline & 63.21 $\pm$ 2.53 & 64.20 $\pm$ 4.46 & 79.46 & 45.79 & 55.79 & 54.91 & 88.65 & 41.61 & 71.68 & 67.77
    \\
    & \ours & \best{71.09 $\pm$ 1.14} & \best{39.95 $\pm$ 7.76} & 83.15 & 49.14 & 75.74 & 70.33 & 90.47 & 44.81 & 82.34 & 72.75
    \\
    \midrule
    \multirow{2}{*}{half-sparse} 
    & baseline & 37.30 $\pm$ 1.32 & 69.67 $\pm$ 2.89 & 61.57 & 8.33 & 57.45 & 50.44 & 60.28 & 23.51 & 17.83 & 18.99
    \\
    & \ours & \best{44.85 $\pm$ 1.03} & \best{62.40 $\pm$ 5.17} & 71.56 & 8.40 & 65.42 & 62.73 & 74.02 & 24.16 & 36.65 & 15.88
    \\
    \bottomrule
    \end{tabular}
    \end{adjustbox}
\end{table}

\paragraph{Implementation details of UNet experiments.}
For the backbone architecture of experiments in \Cref{tab:synapse_sample_eff_unet}, we use a UNet with a ResNet-34 encoder initialized with ImageNet pre-trained weights. We leverage the implementation of UNet and load the pre-trained model via the PyTorch API for segmentation models~\citep{Iakubovskii2019}. For training, we use the same optimizer (SGD with learning rate $0.01$, momentum $0.9$, and weight decay 1e-4) and the total number of epochs (150 epochs on Synapse training set) as the TransUNet experiments, evaluating the last-iteration model on the test dataset. As before, the results in \Cref{tab:synapse_sample_eff_unet} are averages (and standard deviations) over $3$ arbitrary random seeds.

\subsection{Visualization of Segmentation on ACDC dataset} 
As shown in Figure \ref{fig:acdc}, the model trained by \ours segments cardiac structures with more accurate shapes (column 1), identifies organs missed by baseline TransUNet (column 2-3) and circumvents the false-positive pixel classifications (\ie, fake predictions of background pixels as organs) suffered by the TransUNet baseline (column 4-6).

\begin{figure}[ht]
    \centering
    \includegraphics[width=.7\linewidth]{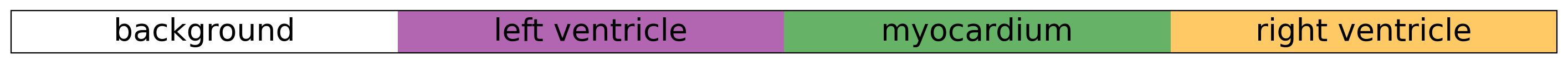}
    \includegraphics[width=.75\linewidth]{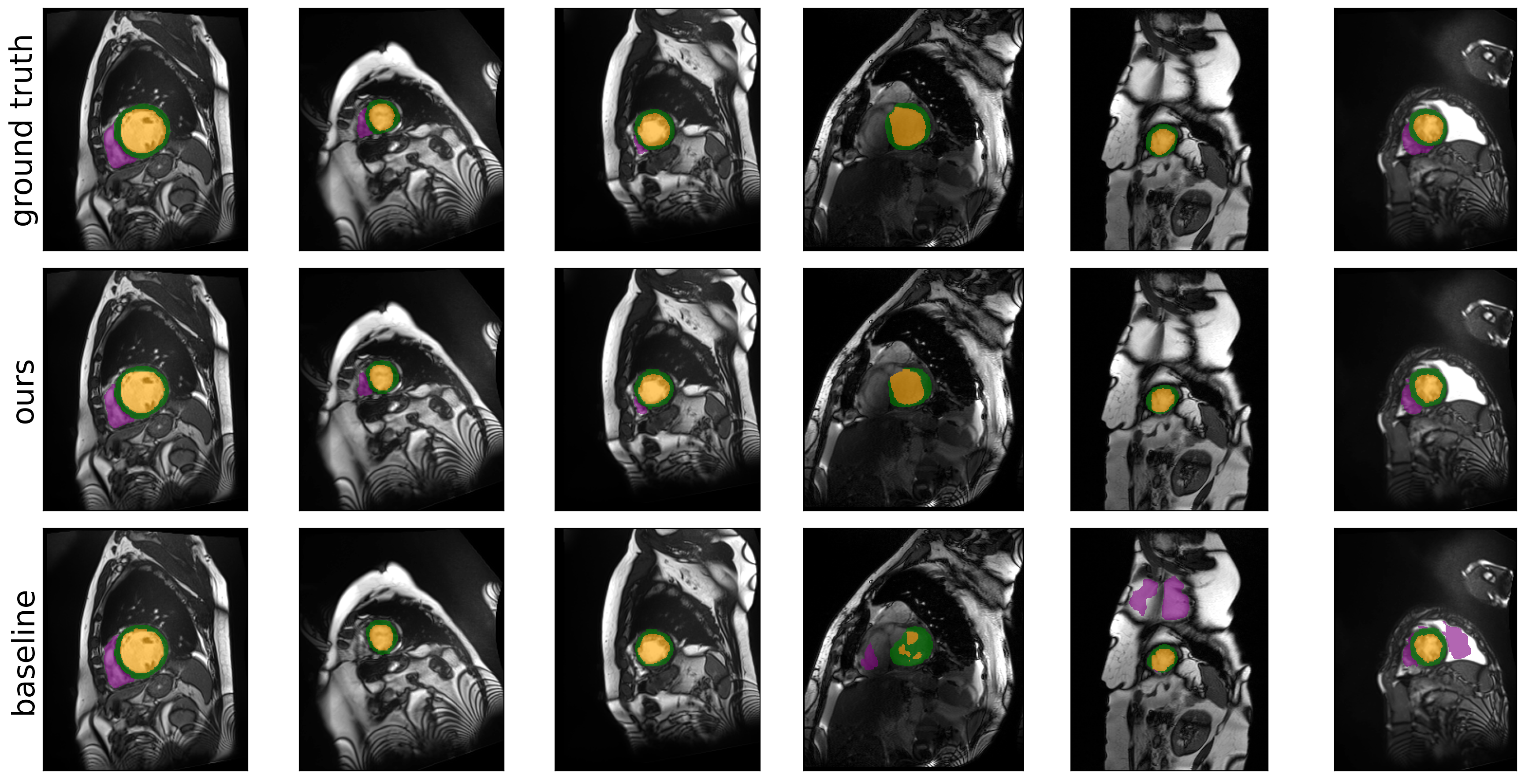}
    \caption{Visualization of segmentation results on ACDC dataset. From top to bottom: ground truth, ours, and baseline method.}
    \label{fig:acdc}
\end{figure}

\subsection{Visualization of Segmentation on Synapse with Distributional Shift} 
\Cref{fig:vis_synapse_half_slice_sparse} visualizes the segmentation predictions on $6$ Synapse test slices made by models trained via \ours(ours) and via the baseline (ERM+SGD) with TransUNet~\citep{chen2021transunet} on the \b{half-sparse} subset of the Synapse training set. We observe that, although the segmentation performances of both the baseline and \ours are compromised by the extreme scarcity of label-dense samples and the severe distributional shift, \ours provides more accurate predictions on the relative positions of organs, as well as less misclassification of organs (\eg, the baseline tends to misclassify other organs and the background as the left kidney). Nevertheless, due to the scarcity of labels, both the model trained with \ours and that trained with the baseline fail to make good predictions on the segmentation boundaries.

\begin{figure}[ht]
    \centering
    \includegraphics[width=.7\linewidth]{figs/synpase_legend.png}
    \includegraphics[width=.7\linewidth]{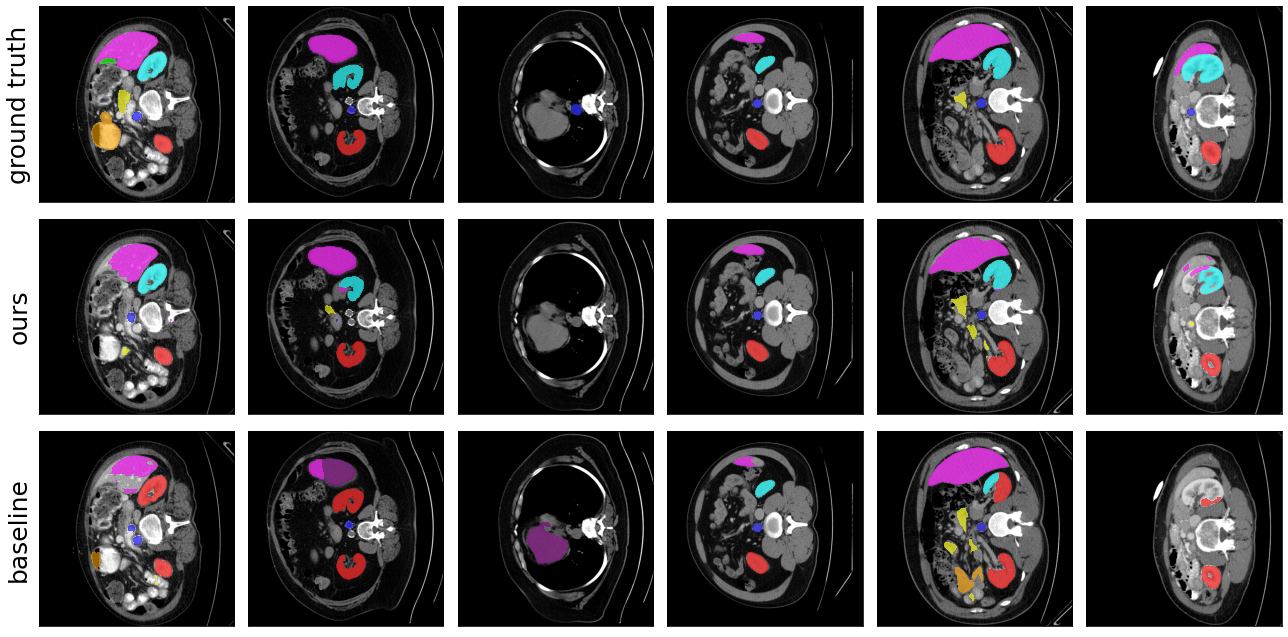}
    \caption{Visualization of segmentation predictions made by models trained via \ours(ours) and via the baseline (ERM+SGD) with TransUNet~\citep{chen2021transunet} on the \b{half-sparse} subset of the Synapse training set. Top to bottom: ground truth, ours (\ours), baseline.}
    \label{fig:vis_synapse_half_slice_sparse}
\end{figure}

\subsection{Experimental Results on Previous Metrics}
In this section, we include the results of experiments on Synapse\footnote{Note that the numbers of correct metrics and metrics in TransUNet~\citep{chen2021transunet} on ACDC dataset are the same.} dataset with metrics defined in TransUNet~\citep{chen2021transunet} for reference. In TransUNet~\citep{chen2021transunet}, DSC is 1 when the sum of ground truth labels is zero (i.e., gt.sum() == 0) while the sum of predicted labels is nonzero (i.e., pred.sum() $>$ 0). However, according to the definition of dice scores, $DSC=2|A \cap B| / (|A|+|B|), \forall A, B$, the DSC for the above case should be 0 since the intersection is 0 and the denominator is non-zero. In our evaluation, we change the special condition for DSC as 1 to pred.sum == 0 and gt.sum() == 0 instead, in which case the denominator is 0.

\begin{table}[ht]
    \caption{\ours with TransUNet trained on the full Synapse and its subsets, measured by metrics in TransUNet~\citep{chen2021transunet}.}
    \centering
    \begin{adjustbox}{width=1\textwidth}  
    \begin{tabular}{ll|cc|cccccccc}
    \toprule
    Training & Method &  DSC $\uparrow$ & HD95 $\downarrow$ & Aorta & Gallbladder & Kidney (L) & Kidney (R) & Liver & Pancreas & Spleen & Stomach 
    \\
    \midrule
    \multirow{2}{*}{full} & baseline &	77.32 &	29.23 &	87.46 &	63.54 &	82.06 &	77.76 &	94.10 &	54.06 &	85.07 &	74.54 
    \\
    & \ours &	80.16 &	25.79 &	87.23 &	63.27 &	84.58 &	81.69 &	94.62 &	58.29 &	90.63 &	81.01 
    \\
    \midrule
    \multirow{2}{*}{half-slice} & baseline & 76.24 &	24.66 &	86.26 &	57.61 &	79.32 &	76.55 &	94.34 &	54.04 &	86.20 &	75.57 
    \\
    & \ours &	78.14 &	29.75 &	86.66 &	62.28 &	81.36 &	78.84 &	94.60 &	57.95 &	85.38 &	78.01 
    \\
    \midrule
    \multirow{2}{*}{half-vol} & baseline &	72.65 &	35.86 &	83.29 &	43.70 &	78.25 &	77.25 &	92.92 &	51.32 &	83.80 &	70.66 
    \\
    & \ours &	75.93 &	34.95 &	84.45 &	60.40 &	79.59 &	76.06 &	93.19 &	54.46 &	84.91 &	74.37 
    \\
    \midrule
    \multirow{2}{*}{half-sparse} & baseline &	0.00 &	0.00 &	0.00 &	0.00 &	0.00 &	0.00 &	0.00 &	0.00 &	0.00 &	0.00 \\
    & \ours &	39.68 &	80.93 &	76.59 &	0.00 &	66.53 &	62.11 &	49.69 &	31.09 &	12.30 &	19.11 \\
    \bottomrule
    \end{tabular}
    \end{adjustbox}
    \label{tab:synapse_sample_eff_prev}
\end{table}

\begin{table}[ht]
    \centering
        \caption{AdaWAC versus hard-thresholding algorithms with TransUNet on Synapse, measured by metrics in TransUNet~\citep{chen2021transunet}.}
    \begin{adjustbox}{width=1\textwidth}  
    \begin{tabular}{l|cc|cccccccc}
    \toprule
    Method &  DSC $\uparrow$ & HD95$\downarrow$ & Aorta & Gallbladder & Kidney (L) & Kidney (R) & Liver & Pancreas & Spleen & Stomach \\
    \midrule
    baseline &	77.32 &	29.23 &	87.46 &	63.54 &	82.06 &	77.76 &	94.10 &	54.06 &	85.07 &	74.54 \\
    trim-train & 77.05 & 26.94 &  86.70 &	60.65 &	80.02 &	76.64 &	94.25 &	54.20 &	86.44 &	77.49 \\
    trim-ratio &	75.30 &	28.59 &	87.35 &	57.29 &	78.70 &	72.22 &	94.18 &	52.32 &	86.31 &	74.03 \\
    \midrule
    trim-train+ACR &	76.70 &	35.06 &	87.11 &	62.22 &	74.19 &	75.25 &	92.19 &	57.16 &	88.21 &	77.30 \\
    trim-ratio+ACR &	79.02 &	33.59 &	86.82 &	61.67 &	83.52 &	81.22 &	94.07 &	59.06 &	88.08 &	77.71 \\
    \ours (ours) &	80.16 &	25.79 &	87.23 &	63.27 &	84.58 &	81.69 &	94.62 &	58.29 &	90.63 &	81.01 \\
    \bottomrule
    \end{tabular}
    \end{adjustbox}
    \label{tab:trim_prev}
\end{table}

\begin{table}[ht]
    \caption{Ablation study of AdaWAC with TransUNet trained on Synapse, measured by metrics in TransUNet~\citep{chen2021transunet}.}
    \centering
    \begin{adjustbox}{width=1\textwidth}  
    \begin{tabular}{l|cc|cccccccc}
    \toprule
    Method &  DSC $\uparrow$ & HD95$\downarrow$ & Aorta & Gallbladder & Kidney (L) & Kidney (R) & Liver & Pancreas & Spleen & Stomach \\
    \midrule
    baseline &	77.32 &	29.23 &	87.46 &	63.54 &	82.06 &	77.76 &	94.10 &	54.06 &	85.07 &	74.54 \\
    reweight-only & 77.72 &	29.24 &	86.15 &	62.31 &	82.96 &	80.28 &	93.42 &	55.86 &	85.29 &	75.49 \\
    ACR-only &	78.93 &	31.65 &	87.96 &	62.67 &	81.79 &	80.21 &	94.52 &	60.41 &	88.07 &	75.83 \\
    \ours-0.01 &	78.98 &	27.81 &	87.58 &	61.09 &	82.29 &	80.22 &	94.90 &	55.92 &	91.63 &	78.23 \\
    \ours-1.0 &	80.16 &	25.79 &	87.23 &	63.27 &	84.58 &	81.69 &	94.62 &	58.29 &	90.63 &	81.01 \\
    \bottomrule
    \end{tabular}
    \end{adjustbox}
    \label{tab:synapse_prev}
\end{table}

\end{document}